\documentclass[journal]{IEEEtran}
\IEEEoverridecommandlockouts

\usepackage{cite}
\usepackage{amsmath,amssymb,amsfonts}
\usepackage{algorithmic}
\usepackage{graphicx}
\usepackage{textcomp}
\usepackage{bm}
\usepackage{mathtools}
\usepackage{subcaption}
\usepackage{float}
\usepackage{geometry}
\usepackage{soul}
\usepackage{algorithm}
\usepackage{booktabs}
\usepackage{tikz}
\usetikzlibrary{calc} 
\usetikzlibrary{arrows.meta, positioning, fit}
\usetikzlibrary{positioning, backgrounds}
\usepackage{amsthm}
\usepackage{booktabs}
\usepackage{multirow}
\usepackage{caption}
\usepackage{orcidlink}
\usetikzlibrary{shapes.geometric}

\newtheorem{theorem}{Theorem}
\newtheorem{assumption}{Assumption}
\newtheorem{lemma}{Lemma}
\newtheorem{definition}{Definition}

\hypersetup{
    colorlinks=true,
    citecolor=blue,
    linkcolor=red,   
    urlcolor=blue
}

\def\BibTeX{{\rm B\kern-.05em{\sc i\kern-.025em b}\kern-.08em
    T\kern-.1667em\lower.7ex\hbox{E}\kern-.125emX}}

\begin{document}

\title{
AI-Enabled Image-Based Hybrid Vision/Force Control of Tendon-Driven Aerial Continuum Manipulators 
}
\author{
Shayan Sepahvand, 
Farrokh Janabi-Sharifi, 
and Farhad Aghili\
\thanks{This work was supported by the National Research Council Canada (Grant AI4L-128-1), the Natural Sciences and Engineering Research Council of Canada (NSERC) under Grant 2023-05542, and the NSERC Discovery Grant 2017-06764 (\textit{Corresponding Author: Farrokh Janabi-Sharifi)}.}%
\thanks{Shayan Sepahvand and Farrokh Janabi-Sharifi are with the Department of Mechanical, Industrial, and Mechatronics Engineering, Toronto Metropolitan University, Toronto, Canada. {\tt\small \{shayan.sepahvand, niloufar.amiri, fsharifi\}@torontomu.ca}}%
\thanks{Farhad Aghili is with the Department of Mechanical, Industrial, and Aerospace Engineering (MIAE), Concordia University, Montreal, Quebec H3G 1M8, Canada. {\tt\small farhad.aghili@concordia.ca}}%
}

\maketitle

\vspace{-1cm}

\begin{abstract}
This paper presents an AI-enabled cascaded hybrid vision/force control framework for tendon-driven aerial continuum manipulators based on constant-strain modeling in $SE(3)$ as a coupled system. The proposed controller is designed to enable autonomous, physical interaction with a static environment while stabilizing the image feature error. The developed strategy combines the cascaded fast fixed-time sliding mode control and a radial basis function neural network to cope with the uncertainties in the image acquired by the eye-in-hand monocular camera and the measurements from the force sensing apparatus. This ensures rapid, online learning of the vision- and force-related uncertainties without requiring offline training. Furthermore, the features are extracted via a state-of-the-art graph neural network architecture employed by a visual servoing framework using line features, rather than relying on heuristic geometric line extractors, to concurrently contribute to tracking the desired normal interaction force during contact and regulating the image feature error. A comparative study benchmarks the proposed controller against established rigid-arm aerial manipulation methods, evaluating robustness across diverse scenarios and feature extraction strategies. The simulation and experimental results showcase the effectiveness of the proposed methodology under various initial conditions and demonstrate robust performance in executing manipulation tasks.

\end{abstract}


\section{Introduction}\label{s1}

 Aerial manipulation has emerged as a key capability for robotic systems operating in remote, elevated, or hazardous environments. While early work has largely focused on rigid-arm aerial manipulators (RA-AMs) \cite{Ollero2022}, recent advances in continuum robotics have enabled a new class of systems—tendon-driven aerial continuum manipulators (TD-ACMs)—that combine the mobility of multirotor platforms with the intrinsic compliance of continuum robots (CRs) \cite{Samadikhoshkho2021, Samadikhoshkho2022, Amiri2025}. Owing to their compliant structures, continuum arms offer superior adaptability in cluttered environments, reduced impact forces during contact, and improved safety during aerial physical interaction (APhI) with the environment \cite{Samadikhoshkho2022, Peng2025}.
 
Due to the continuous deformation of the continuum arm along its backbone, the mechanical structure of TD-ACMs fundamentally differs from that of RA-AMs in terms of kinematics, dynamics, force transmission, and interaction behavior. While rigid-arm manipulators can exhibit localized joint compliance and well-defined link inertias \cite{Ollero2022, Bartelds2016, Housseyn2023}, TD-ACMs feature configuration-dependent compliance, strong coupling between shape and force, and deformation under localized external loads \cite{Samadikhoshkho2022}. Therefore, modeling and control of TD-ACMs introduce significant challenges. (i) The distributed curvature and strain of CRs lead to computationally challenging formulations, often represented as Cosserat rod partial differential equations (PDEs), with underlying difficulties in accurate constitutive modeling. For instance, capturing bidirectional and shape-dependent kinematic and dynamic couplings in a time-efficient model of TD-ACMs remains an open challenge, which in turn complicates the design of model-based controllers. (ii) The states (shape and motion) of a deforming arm in TD-ACMs depend strongly on contact forces, and are also adversely affected by aerodynamic forces near contact surfaces (e.g., vortex shedding and downwash effect). Their state sensing cannot be obtained directly (e.g., using joint encoders and forward kinematics), leading to difficult and often uncertain and indirect state sensing and estimation. This, in turn, complicates the localization and positioning of the arm tip with respect to an uncertain target. (iii) Contrary to RA-AMs, which rely on fast UAV attitude loop and slower arm motion with the combined system dynamics treated hierarchically in control design \cite{Kremer2022}, arm deformations in soft aerial manipulators occur on similar time scales as UAV attitude dynamics \cite{Jitosho2025}, thereby challenging the use of hierarchical control schemes often employed in RA-AMs. (iv) Since accurate arm models are computationally intractable for real-time implementation and simplified models may fail to preserve the stability margin, the model-control compromise is far more severe in TD-ACMs. Stability dependence on shape, contact, disturbances, and modeling errors can cause loss of system control. Therefore, robustness becomes a central objective in the design of controllers for TD-ACMs. In short, these characteristics make interaction control for TD-ACMs substantially more challenging than for their rigid counterparts, and lead us to focus on effective modeling, efficient computing, and robust control design, augmented with a remedy for the state estimation and localization issue, as the core components of our interaction control strategy.  

\subsection{Related Work}

Control strategies for TD-ACMs have primarily focused on free-flight motion control \cite{Samadikhoshkho2021, Samadikhoshkho2022, Amiri2025, Peng2025, Ghorbani2023}, while interaction control of TD-ACMs has remained comparatively unexplored. Aerial interaction control has been more extensively studied in the context of RA-AMs \cite{Ollero2022}, with approaches commonly categorized into impedance control \cite{Suarez2018, Marković2021} and hybrid position/force control \cite{Nava2020, Tzoumanikas2020}. While impedance control is attractive for its simplicity and smooth transitions between free-flight and contact, it typically exhibits reduced tracking accuracy \cite{Ollero2022}. Alternatively, hybrid position/force control provides superior tracking accuracy, yet requires reliable localization \cite{Ollero2022}. Therefore, the focus of our work will be on designing the first, to the best of our knowledge, hybrid position/force control design for TD-ACMs with a remedy for the localization issue.  

Most prior aerial interaction control schemes rely on external localization sources such as global positioning system (GPS) or simultaneous localization and mapping (SLAM). However, these approaches become unreliable near infrastructure due to signal degradation and drift, which is particularly problematic for high-precision manipulation where vehicle and manipulator motions are tightly coupled \cite{He2023}. Owing to the advantages of vision sensors (e.g., their light weight and low power usage) in aerial robotics, visual servoing (VS) offers an appealing alternative to overcome modeling and localization uncertainties \cite{Corke2011}. In particular, image-based visual servoing (IBVS) circumvents full pose estimation, reduces sensitivity to calibration errors, and lowers computational burden when compared to alternative VS techniques such as position-based visual servoing (PBVS) \cite{Corke2011, Sharifi2011, Sepahvand2024}.  

Several vision-based interaction control methods have been reported for RA-AMs \cite{Byun2024, Xu2023, Wu2024, Zeng2023}. For example, in \cite{Xu2023}, a multi-stage adaptive visual impedance control strategy was developed to regulate contact force on a planar surface; however, the proportional-derivative (PD) controller utilized in the vision loop is endowed with image noise amplification, leading to chattering in the commanded camera velocity. A major drawback of many of those approaches \cite{Xu2023, Wu2024, Zeng2023} is that they typically assume point features and reliable feature extraction, assumptions that are often difficult to meet in real-world applications. Their assumptions of negligible aerodynamic effect (for rigid arms) and rigid-body kinematics cannot be extended to TD-ACMs either, and the deformable nature of CRs introduces additional uncertainties in kinematics, interaction matrix, feature measurements, and force sensing, all of which are amplified during contact. Therefore, our focus will be on the development of a robust hybrid vision/force control approach integrated with a reliable and computationally efficient IBVS scheme.  

\subsection{Approach and Contributions}

The proposed approach is fundamentally distinct from the previous work, with the main contributions summarized as follows. 

(1) Development of the first hybrid vision/force interaction control strategy for TD-ACMs. The proposed coupled control framework employs a cascaded fast fixed-time sliding mode controller (CFTSMC), augmented with radial basis function neural network (RBFNN) estimators, and an IBVS scheme. The robust and computationally efficient control design ensures robustness to modeling and measurement uncertainties, such as those in the image feature extraction, the interaction matrix, and force measurements, while guaranteeing fixed-time convergence of the image feature and force errors.

(2) Design of a robust and efficient VS scheme for APhI control. Our work is among a few studies that integrate the advantages of VS—such as robustness to modeling and positioning uncertainties— into an interaction control framework. It is further distinguished by enhanced perceptual robustness and computational efficiency through the use of line features, an RBFNN-based feature error estimator, and a deep feature extraction technique based on an offline-trained graph neural network (GNN).  The incorporation of perceptually robust line features, rather than point features, image moments, or keypoints \cite{Xu2023, Wu2024, Amiri2024CASE}, enhances robustness to viewpoint variations, partial occlusions, and illumination changes. Moreover, line features are more abundant in man-made environments \cite{Qu2025}. The computational burden and uncertainty associated with line feature extraction are mitigated through the adoption of a data-driven approach leveraging connectivity information \cite{Pautrat2023}, enabling interaction in unstructured environments without relying on man-made markers utilized in many previous works \cite{Xu2023, Zeng2023}. 

(3) Formulation of efficient and accurate kinematic and dynamic models in the special Euclidean group, $SE(3)$. A constant-strain modeling approach \cite{Renda2018} is adopted to formulate coupled kinematic and dynamic models in $SE(3)$ for TD-ACMs. The developed models capture vehicle–arm coupling and deformation effects while remaining tractable for effective simulation and control design. Unlike prior coupled models based on piecewise constant-curvature (PCC) formulations in the special orthogonal group ($SO(3)$) \cite{Samadikhoshkho2021, Ghorbani2023}, which suffer from limited accuracy in force-interaction tasks, or Cosserat rod models \cite{Jalali2022}, which are computationally intensive, the proposed framework achieves an effective balance between accuracy and computational efficiency. Furthermore, the formulation in $SE(3)$—whose advantages are well documented \cite{Seo2025}—enables a compact representation suitable for hybrid vision/force control. 

(4) Extensive software-in-the-loop and hardware-in-the loop validation. Extensive simulation and experimental studies are conducted to evaluate the approach robustness and compare its performance against conventional control schemes commonly used in RA-AMs. For instance, the proposed
controller out-performs PhI control methods prevalent in the RA-AM field such as proportional integral/proportional derivative (PI/PD) and the cascaded integral sliding mode controller (CISMC) \cite{Ahmadi2022} in tracking performance, demonstrated by improved root mean square error (RMSE) and standard deviation (STD) of error.

\section{Problem Formulation}

\subsection{Notation and Coordinate Frames}
The schematics of the TD-ACM along with the key coordinate frames are illustrated in Fig.~\ref{fig:frame_attachment}. The world frame (inertial frame) is defined as $\mathcal{F}_I=O_I\{\bm x_I, \bm y_I, \bm z_I\}$, whereas the body frame $\mathcal{F}_u=O_u\{\bm x_u, \bm y_u, \bm z_u\}$, is attached to the center of mass of the aerial platform so that the local inertia tensor is a diagonal matrix. Following the discretization of the arm into finite points, the frame  $\mathcal{F}_{a_i}=O_{a_i}\{\bm x_{a_i}, \bm y_{a_i}, \bm z_{a_i}\}$ is associated with the $i^{th}$ significant point, where \linebreak $i=\{1,...,n\}$, lying on the CR backbone. The eye-in-hand monocular camera frame $\mathcal{F}_C=O_C\{\bm x_C, \bm y_C, \bm z_C\}$ is attached to the optical center, assumed to be fixed w.r.t the tip frame $\mathcal{F}_T=O_T\{\bm x_T, \bm y_T, \bm z_T\}$. The wrench measurement apparatus, associated with $\mathcal{F}_f=O_f\{\bm x_f, \bm y_f, \bm z_f\}$, can be either mounted on the tip or the workpiece.

The notation $(\cdot)^\dagger$ denotes the Moore–Penrose generalized inverse of a matrix. A left-hand superscript is adopted to indicate the transformation between two frames, e.g. $\prescript{B}{}{\bm{g}}^{}_{A}$ transforms from frame $\mathcal{F}_A$ to $\mathcal{F}_B$. Notations $(\cdot)^\prime$ and $\dot{(\cdot)}$ indicate the derivative w.r.t the spatial and time variables $X$ and $t$, respectively. The operation $\hat{(\cdot)}$ is an isomorphism from the vector to the matrix representation of the Lie algebra $\mathfrak{se}(3)$. 
$\bm{\bm{\xi}} =    
\begin{bmatrix}
    \bm{\xi}_{a}^\top, &
    \bm{\xi}_{l}^\top
    \end{bmatrix}^\top \in \mathbb{R}^6$ constitutes the two sub-components $\bm{\xi}_{a}\in \mathbb{R}^3$ and $\bm{\xi}_{l}\in \mathbb{R}^3$, respectively. The $\tilde{(.)}$ operator denotes a mapping from $\mathbb{R}^3$ to Lie {algebra} $\mathfrak{so}(3)$ having $\bm{\xi}_a = \begin{bmatrix}
    \xi_{a_x}, &
    \xi_{a_y}, &
    \xi_{a_z}
\end{bmatrix}^\top \in \mathbb{R}^3$,
and the mappings $\operatorname{Ad}_{(\cdot)}$ and $\operatorname{ad}_{(\cdot)}$ represent the adjoint representation mappings of the Lie group and the Lie algebra, respectively. The notation $\operatorname{sig}(\bm{\gamma})^{\kappa}$ denotes the elementwise 
generalized sign function, defined as
$
\operatorname{sig}(\bm{\gamma})^{\kappa}
= \big[
|\gamma_1|^{\kappa}\operatorname{sign}(\gamma_1), \dots, 
         |\gamma_n|^{\kappa}\operatorname{sign}(\gamma_n) \big]^\top,
$
for any $\bm{\gamma} = [\gamma_1, \dots, \gamma_n]^\top \in \mathbb{R}^n$, where 
$\operatorname{sign}(\cdot)$ is the standard signum function. 

\subsection{Kinematics of the CR}

The forward kinematics of a single-section CR, assuming constant strain, is presented below. The pose of an infinitesimal cross-section along the backbone is defined as:

\begin{equation}
\bm{g}=
    \begin{bmatrix}
    \bm{R}(X,t) & \bm{u}(X,t)\\
    \bm{0}_{1 \times 3} & 1
    \end{bmatrix} \in SE(3),
\end{equation}

\noindent where $X$ is the spatial variable, $t$ denotes time,\linebreak $\bm{R}(X,t) \in SO(3)$ and $\bm{u}(X,t) \in \mathbb{R}^3$ are the rotation matrix and displacement vector, respectively. The material curvilinear abscissa $X$ is upper bounded by the total length of the arm $l_c$, i.e., $X \in \left[0,\ l_c \right]$. For the sake of compactness, the dependence on $(X, t)$ is dropped from the variables. The strain vector field $\bm{\xi}$ is defined as below: 

\begin{figure}[t]
    \centering
    \includegraphics[scale = 0.2]{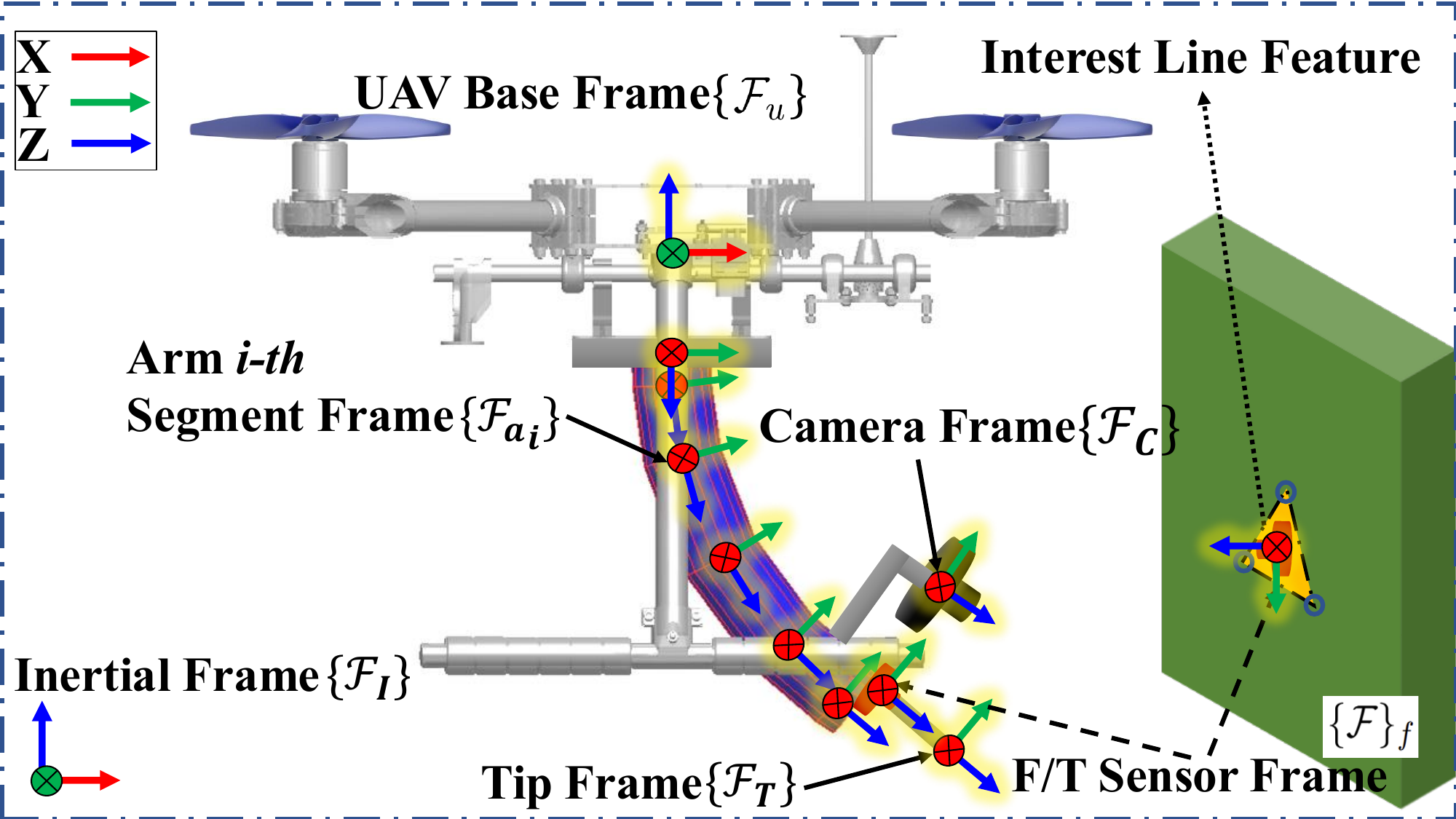}
    \caption{The schematic representation of the TD-ACM and the key coordinate frames.}
    \label{fig:frame_attachment}
\end{figure}

\vspace{-0.3cm}

\begin{equation}
    \hat{{\bm{\xi}}} = \bm{g}^{-1} \frac{\partial \bm{g}}{\partial X} = \bm{g}^{-1}   \bm{g}^\prime \in \mathfrak{se}(3).
\label{eq:main_strain}
\end{equation}

\noindent Likewise, the twist vector field $\bm{\varrho} =    \begin{bmatrix}
    \bm{\varrho}_{a}^\top,&
    \bm{\varrho}_{l}^\top
    \end{bmatrix}^\top \in \mathbb{R}^6$ is defined leveraging the time derivative of the configuration matrix:

\vspace{-0.5cm}

\begin{equation}
    \hat{\bm{\varrho}} = \bm{g}^{-1} \frac{\partial \bm{g}}{\partial t} = \bm{g}^{-1}   \dot{\bm{g}}\in \mathfrak{se}(3).
\end{equation}

\noindent Since the mixed partial derivatives commute, by employing \eqref{eq:main_strain}, the following compatibility equation is derived:\linebreak
$
\bm{\varrho}^\prime = \dot{\bm{\xi}}-\operatorname{ad}_{\bm{\xi}} \bm{\varrho}.
\label{eq:twist_equation}
$
\subsection{Kinematics of the TD-ACM}

The differential kinematics is expressed by incorporating the kinematic coupling between the arm and the UAV. Given $\mathfrak{n}(\bm \xi, \bm q_t, X) = \bm 0$ with $\mathfrak{n}$ defining the strain at any point along the CR, and assuming $\det(\partial \mathfrak{\bm n}/\partial \bm \xi) \neq 0$, the implicit function theorem~\cite{apostol1974} implies the existence of the mapping $\bm \xi = \bm \xi(\bm q_t)$ under the constant strain assumption. Let the generalized coordinate vector $\bm \xi = \bm{q}_t$ be defined as
$\bm{q}_t = 
\begin{bmatrix}
\bm{q}_{u}^\top & \bm{q}_{a}^\top
\end{bmatrix}^\top \in \mathbb{R}^{n_u+n_a}$, \linebreak where $\bm{q}_{u}\in \mathbb{R}^{n_u}$ and $\bm{q}_{a}\in \mathbb{R}^{n_a}$ represent the joint variables of the UAV and the CR, respectively. The configuration of the aerial platform is obtained w.r.t the inertial frame as $ \prescript{I}{}{\bm{g}}^{}_{u} = \bm{g}_s \exp \left(\hat{\bm{q}}_u\right)$, where $\bm{g}_s$ is the initial static transformation between the UAV body frame and the inertial frame. Considering the Jacobian expressed in the local frame $\mathcal{F}_u$, i.e., $\operatorname{Ad}_{\bm{g}_u^{-1}}  \bm{T}_u = \prescript{u}{}{\bm{J}}^{}_{u}$, the local body twist $\prescript{u}{}{\bm{\varrho}}^{}_{u}$, given the initial twist $\prescript{u}{}{\bm{\varrho}}^{}_{u_{0}}$, is given by 
 $\prescript{u}{}{\bm{\varrho}}^{}_{u} = \operatorname{Ad}_{\bm{g}_u^{-1}} \left ( \bm{T}_u \dot{\bm{q}}_u +\prescript{u}{}{\bm{\varrho}}^{}_{u_{0}} \right )$. The time derivative of the local Jacobian is calculated as
 $\prescript{u}{}{\dot{\bm{J}}}^{}_{u} = \operatorname{ad}_{\prescript{u}{}{\bm{\varrho}}^{}_{u}}  \bm{T}_u \dot{\bm{q}}_u +\dot{\bm{T}}_u
$. Considering the static transformation between the UAV frame and the soft arm base, one can formulate: $
   \prescript{I}{}{\bm{g}}^{}_{a_{0}} = \prescript{I}{}{\bm{g}}^{}_{u}  \prescript{}{}{\bm{g}}^{}_{a_{s}}$, and
$\prescript{a_0}{}{\bm{\varrho}}^{}_{a_0} = \operatorname{Ad}_{\bm{g}_{a_s}^{-1}} \prescript{u}{}{\bm{\varrho}}^{}_{u}
$.

The arm is next discretized into several significant points along its abscissa using Gaussian quadrature. The $i^{th}$ significant point kinematics is calculated based on the following relations:
\vspace{-0.2cm}
\begin{equation}
   \prescript{I}{}{\bm{g}}^{}_{a_i} = \prescript{I}{}{\bm{g}}^{}_{a_{i-1}} \exp \left( h_i\left( 
  \bm{B}\bm{q}_{a} +\bm{\xi}_{a}^* \right)\right),
\end{equation}
\vspace{-0.5cm}
\begin{equation}
 \prescript{a_i}{}{\bm{\varrho}}^{}_{a_i} = \operatorname{Ad}_{\prescript{a_{i-1}}{}{\bm{g}}^{-1}_{a_i}} \left ( h_i\bm{T}_{a_i} \bm{B}\bm{\dot{q}}_{a} +\prescript{a_{i-1}}{}{\bm{\varrho}}^{}_{a_{i-1}} \right ),
 \label{eq:twist_SE3}
\end{equation}
\vspace{-0.5cm}
\begin{equation}
\prescript{a_i}{}{\bm{J}}^{}_{a_i} = \operatorname{Ad}_{\bm{g}_{a_{i}}^{-1}}  \left(h_i\bm{T}_{a_i} \bm{B}+\prescript{a_{i-1}}{}{\bm{J}}^{}_{a_{i-1}}\right),
\label{eq:jacobian_general}
\end{equation}
\vspace{-0.5cm}
\begin{equation}
\prescript{a_i}{}{\bm{\dot{J}}}^{}_{a_i} = \operatorname{Ad}_{\bm{g}_{a_{i}}^{-1}}  \left(h_i\operatorname{ad}_{ \prescript{a_i}{}{\bm{\varrho}}^{}_{a_i}}\bm{T}_{a_i} \bm{B}+h_i\bm{\dot{T}}_{a_i}\bm{B}+\prescript{a_{i-1}}{}{\dot{\bm{J}}}^{}_{a_{i-1}}\right),
\end{equation}

\noindent where $\bm{B}$ is a matrix that gives information about the controllable {DoF} defined for the arm \cite{Renda2024}. Therein, $h_i$ is a scalar defining the distance between two significant points that discretize the arm. It is worth highlighting that the dynamics of TD-ACM is derived from the above kinematic equations and serves as a backbone for our simulation studies.

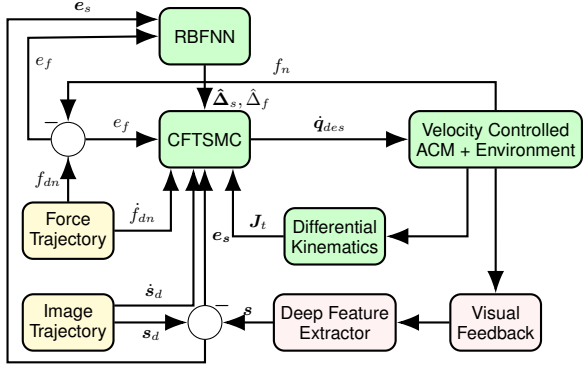
\begin{figure}[t]
\centering
\begin{tikzpicture}[
    scale=0.75, transform shape,
    font=\sffamily\small, 
    node distance=0.8cm and 0.8cm, 
    block/.style={draw, thick, minimum width=1.6cm, minimum height=1cm, align=center, fill=white, rounded corners},
    deriv/.style={draw, thick, minimum width=1cm, minimum height=1cm, align=center, fill=white, rounded corners},
    sum/.style={draw, circle, inner sep=1pt, minimum size=6mm},
    arrow/.style={-{Latex[length=3mm]}, thick},
    line/.style={thick},  
    mylabel/.style={font=\footnotesize},
    background/.style={rectangle, fill=#1, rounded corners, inner sep=0.2cm}
]


\node[sum] (sum1) {};

\node[block, below=of sum1, fill=yellow!20] (ForceTrajectory) {Force\\Trajectory};
\node[block, below=of ForceTrajectory, fill=yellow!20, yshift=0.2cm] (ImageTrajectory) {Image\\Trajectory};

\node[block, right=of sum1, fill=green!20, xshift=5mm] (pi) {CFTSMC};
\node[block, above=of pi, fill=green!20] (rbfnn) {RBFNN};

\node[block, right=of pi, xshift=2cm, fill=green!20] (plant) {Velocity Controlled \\ ACM + Environment};

\node[sum, right=of ImageTrajectory, xshift=5mm] (OutComp) {};
\node[block, right=of OutComp, xshift=1mm, fill=pink!20] (extractor) {Deep Feature\\ Extractor};

\node[block, at=(plant|-extractor), fill=pink!20] (feedback) {Visual\\Feedback};

\node[block, above=of extractor, fill=green!20, yshift=-3mm] (jacobian) {Differential\\ Kinematics};

\node at ([xshift=-0.3cm, yshift=0.3cm]sum1.center) {$\bm{-}$};
\node at ([xshift=0.3cm, yshift=0.3cm]OutComp.center) {$\bm{-}$};


\draw[arrow] (pi) -- node[midway, above] {$\dot{\bm{q}}_{des}$}(plant);
\draw[arrow] (sum1) -- node[midway, above] {$e_{f}$} (pi);
\draw[arrow] (OutComp) -- node[midway, right] {$\bm{e_s}$} (pi); 

\draw[arrow] (ForceTrajectory) -- (sum1);
\draw (ForceTrajectory) -- node[midway, left] {$f_{dn}$} (sum1); 
\draw[arrow] (ForceTrajectory.east) -| node[near start, above] {$\dot{f}_{dn}$} ($(pi.south)+(-0.6,0)$);

\draw[arrow] (ImageTrajectory) -- node[midway, below] {$\bm{s}_d$} (OutComp);
\draw[arrow] ($(ImageTrajectory.east)+(0,0.3)$) -| node[near start, above] {$\dot{\bm{s}}_d$} ($(pi.south)+(-0.2,0)$);

\draw[arrow] (plant.north) -- ++(0,0.5) -| node[pos=0.25, above] {$f_{n}$}  (sum1.north);

\draw[arrow] (plant.south) -- (feedback.north); 
\draw[arrow] (feedback) -- (extractor);
\draw[arrow] (extractor) -- node[midway, above] {$\bm{s}$} (OutComp);

\draw[arrow] (jacobian.west) -| node[near start, above] {$\bm J_t$} ($(pi.south) + (0.5, 0)$);
\draw[arrow] ($(plant.south) + (-0.5,0)$) |- (jacobian.east);

\draw[arrow] (sum1.west) -- ++(-0.4,0) 
  -- ++(0,1.8) 
  node[near end, right] {$e_f$} 
  -- (rbfnn.west);

\draw[arrow] (OutComp.south) -- ++(0,-0.4) 
  -- ++(-3.5,0) 
  |- node[near end, above] {$\bm e_s$} 
  ($(rbfnn.west)+(0,0.3)$);

\draw[arrow] (rbfnn) -- node[near end, right] {$\bm {\hat {\Delta}}_s, \hat \Delta_f$} (pi.north);

\end{tikzpicture}
\caption{The block diagram of the proposed control scheme.}
\label{fig:block_diagram}
\end{figure}

\section{Controller Design}

The ensuing formulation details a novel control system for the TD-ACM capable of maintaining a stable contact force exerted on a static planar surface through the end-effector. The overall structure of this control strategy is depicted in Fig.~\ref{fig:block_diagram}.

\subsection{Visual Servoing Based on Line Features}

To derive the IBVS control command, the $N_l$ tracked line parameters are aggregated into the vector \linebreak $\bm{s} = \left[\theta_1,\rho_1, \dots,\theta_i,\rho_i,\dots, \theta_{N_l},\rho_{N_l}\right]^\top \in \mathbb{R}^{2N_l}$. The interaction matrix corresponding to these features can be written as:

\begin{equation}
    \dot{\bm{s}} = \prescript{c}{}{\bm{L}}\prescript{t}{}{\bm{\varrho}_c}+\bm \Delta_s,
    \label{feature_error_rate}
\end{equation}

\noindent where $\bm{\Delta}_s\in \mathbb{R}^{2N_l}$ represents the uncertainty in the image features, $\prescript{c}{}{\bm{L}} = \left(\bm{L}_1^\top, \bm{L}_2^\top, \dots,\bm{L}_i^\top,\dots, \bm{L}_{N_l}^\top \right)^\top \in \mathbb{R}^{2N_l \times 6}$ is the vertical stacking of the matrices $\bm{L}_i = \begin{bmatrix} \bm{L}_{i,1}, & \bm{L}_{i,2} \end{bmatrix},
$ given by \eqref{eq:interaction_matrix},

\begin{equation}
    \begin{aligned}
        \bm{L}_{i,1} &= 
        \begin{bmatrix}
             \rho_i \sin{\theta_i} & \rho_i \cos{\theta_i} & 1 \\
             \cos{\theta_i}(1 + \rho_i^2) & -\sin{\theta_i}(1 + \rho_i^2) & 0
        \end{bmatrix}, \\
        \bm{L}_{i,2} &= 
        \begin{bmatrix}
             \lambda_{\theta_i} \sin{\theta_i} & \lambda_{\theta_i} \cos{\theta_i} & -\rho_i \lambda_{\theta_i} \\
             \lambda_{\rho_i} \sin{\theta_i} & \lambda_{\rho_i} \cos{\theta_i} & -\lambda_{\rho_i} \rho_i
        \end{bmatrix},
    \end{aligned}
\label{eq:interaction_matrix}
\end{equation}

\noindent in which the parameters \(\lambda_{\theta_i}\) and \(\lambda_{\rho_i}\) are defined as below:

\begin{align}
    \lambda_{\theta_i} &= \frac{-a_p\cos\theta_i + b_p\sin\theta_i}{d_p}, \\
    \lambda_{\rho_i}   &= \frac{a_p\rho_i\sin\theta_i + b_p\rho_i\cos\theta_i + c_p}{d_p},
\end{align}

\noindent where, for the $i^{th}$ line, $\theta_i$ represents the angle of the line relative to the horizontal image axis, while $\rho_i$ denotes the perpendicular distance from the image frame origin to the line. The equation of the plane containing the line feature in 3-space can be given as $a_pX_w+b_pY_w+c_pZ_w+d_p=0$. 

Let $\mathcal{F}_{a_n}$ denote the frame attached to the last significant cross-section. Expressing the local twist of the tip in the tip frame via \eqref{eq:twist_SE3} yields:

\vspace{-0.5cm}

\begin{equation}
 \prescript{t}{}{\bm{\varrho}}^{}_{t} =\operatorname{Ad}_{\prescript{t}{}{\bm{g}}_{a_n}}  \operatorname{Ad}_{\prescript{a_{n-1}}{}{\bm{g}}^{-1}_{a_n}} \left ( h_n\bm{T}_{a_n} \bm{B}\bm{\dot{q}}_{a} +\prescript{a_{n-1}}{}{\bm{\varrho}}^{}_{a_{n-1}} \right ),
\end{equation}

\noindent where $\prescript{t}{}{\bm{g}}_{a_n}$ is a fixed known transformation between the tip and the last significant point on the CR. Considering the static transformation between the tip and the camera frame, $\prescript{c}{}{\bm{g}}_t$, the local twist of the camera can be obtained as
$
 \prescript{t}{}{\bm{\varrho}}^{}_{c} = {\operatorname{Ad}_{\prescript{c}{}{\bm{g}}_t}} \prescript{t}{}{\bm{\varrho}}^{}_{t}.
$ Using \eqref{eq:interaction_matrix}, the conventional IBVS control law to stabilize the image feature error, i.e., $\bm e_s \triangleq \left( \bm{s}_d - \bm{s} \right)$ can be presented as,
$
\bm{\dot{q}}_{des} =\bm{K}_{p} \bm{J}_t^{\dagger} \,{\operatorname{Ad}_{\prescript{t}{}{\bm{g}}_c}} \prescript{c}{}{\bm{L}}^{\dagger}\bm e_s
$ \cite{Corke2011}.


\subsection{Gluestick Line Feature Extraction }
 We have adopted the Gluestick \cite{Pautrat2023} framework to address specific control requirements, which are the real-time extraction and parameterization of a finite number of lines. Gluestick is a deep line segment detector that outperforms many traditional techniques rooted in Hough transform relying on the detection of high-gradient image regions. This technique takes advantage of the descriptor-based matching using SuperPoints \cite{DeTone2018} interest corners, integrating both endpoints and the lines simultaneously to better disambiguate between the line segments. In this architecture, robot operating system (ROS) middleware is leveraged to publish the parameterized feature vector from Gluestick to the controller. 

\subsection{Interaction Force Modeling}
During contact between the tip and the workpiece, the normal exerted force aligned with the z-axis of the tip, i.e., $\bm z_T$, is computed as $f_n = \hat{k}_e p_t$ for $p_t \geq 0$ and $0$ otherwise, where $\hat{k}_e$ is the estimated stiffness of the environment, and $p_t$ is the displacement of tip along its local z-axis. By taking the time derivative of normal interaction force in case of contact, we obtain:

\begin{equation}
\dot{f}_n =\hat{k}_e \bm{Q}\prescript{t}{}{\bm{\varrho}_c}+\Delta_f,
\label{eq:force_rate}
\end{equation}

\noindent with $\bm{Q} = \begin{bmatrix}
    0&0&0&0&0&1
\end{bmatrix}$
which is known as the direction matrix and $\Delta_f$ representing the uncertain part of the force system. Substituting \eqref{feature_error_rate} into \eqref{eq:force_rate} yields:

\begin{equation}
    \dot{f}_n =\hat{k}_e \bm{Q}\prescript{c}{}{\bm{L}}^{\dagger}\dot{\bm{s}}+\Delta_f
    \label{eq:StiffnessModel2}.
\end{equation}

\begin{assumption}
Similar to the previous works \textup{\cite{Kirner2024, Xu2023}}, we assume negligible friction because it has little impact on the TD-ACM to accomplish the manipulation task.
\label{assum:friction_force}
\end{assumption}

\begin{assumption}
The values of $\bm \Delta_s$ and $\Delta_f$ in \eqref{feature_error_rate} and \eqref{eq:force_rate}, respectively, are bounded due to the finite bandwidth of the sensing hardware.
\label{assum:uncer_bound}
\end{assumption}

\subsection{SE(3)-Based Cascaded Vision/Force Control with RBFNN Estimation}

The hybrid vision/force control is proposed to achieve asymptotic convergence of both visual and the normal force tracking errors:
\begin{equation}
\lim_{t \to \infty} \bm{e}_s = \bm{0}, \quad \lim_{t \to \infty} e_{f} = 0,
\end{equation}
where
\begin{equation}
\bm{e}_s = \bm{s}_d - \bm{s}, \quad e_{f} = f_{dn} - f_n,
\label{eq:error_main}
\end{equation}
with $\bm{e}_s \in \mathbb{R}^{2N_l}$ and $e_{f} \in \mathbb{R}$. The desired values of the image feature vector and the normal force profile are denoted by $\bm{s}_d$ and $f_{dn}$, respectively. The relation between the wrench measured by the force-torque (F/T) sensing apparatus, $\bm{F}_f$, and the variable $f_n$ can be stated as:

\begin{equation}
   f_n = \bm Q \bm{F}_t=\bm Q \operatorname{Ad}^*_{\prescript{t}{}{\bm{g}_f}}\bm{F}_f,
\end{equation}

\noindent where $\operatorname{Ad}^*_{(\cdot)}$ is the co-adjoint representation of the configuration matrix, which depends on how the sensor is integrated into the system (see Fig.~\ref{fig:frame_attachment}). The system dynamics is given by

\begin{equation}
\begin{aligned}
\dot{\bm{e}}_s &= \dot{\bm{s}}_d - \bm{L} \bm{J}_t \dot{\bm{q}} - \bm{\Delta}_s, \\
\dot{e}_{f} &= \dot{f}_{dn} - \hat{k}_e \bm{Q} \bm{L}^\dagger \dot{\bm{s}} - \Delta_f.
\end{aligned}
\label{eq:uncertain_error_system}
\end{equation}

To prove the stability of the proposed controller, the following definitions and assumptions are established.

\begin{definition} 
If the origin of the system $\dot{\Lambda} = \mathfrak{s}(\Lambda)$ with $\mathfrak{s}(0) = 0$ is Lyapunov stable, then the origin is said to be finite-time stable if there exists a settling-time function $T: \mathbb{R}^n \setminus \{0\} \to \mathbb{R}^+$ such that the solution $\Lambda(t, \Lambda_0)$ of the system satisfies \textup{\cite{Bhat2000}}
\[
\lim_{t \to T(\Lambda_0)} \Lambda(t, \Lambda_0) = 0, 
\quad \forall \Lambda_0 \in \mathbb{R}^n \setminus \{0\}.
\]
\end{definition}

\begin{definition}
The system in Definition 1 is said to be \emph{fixed-time stable} if it is globally 
finite-time stable, meaning that all trajectories converge to the origin within 
a bounded convergence time $T(\Lambda_0)$. Furthermore, there exists a positive 
constant $T_{\max}$ such that \textup{\cite{Polyakov2012}}
\[
T(\Lambda_0) < T_{\max}, \quad \forall \Lambda_0 \in \mathbb{R}^n.
\]
\end{definition}

\begin{lemma}
Suppose there exists a continuously differentiable Lyapunov function $V(\Lambda): \mathbb{R}^n \to \mathbb{R}^+$ and positive constants $\mathcal{L}_1, \mathcal{L}_2, \mathcal{B}_1, \mathcal{B}_2, {\mathcal{B}_3} \in \mathbb{R}^+$ satisfying
$
0 < \mathcal{B}_1\mathcal{B}_3< 1, \quad \mathcal{B}_2\mathcal{B}_3 > 1,
$
such that
$
\dot{V}(\Lambda) \leq -(\mathcal{L}_1 V(\Lambda)^{\mathcal{B}_1} + \mathcal{L}_2 V(\Lambda)^{\mathcal{B}_2})^{\mathcal{B}_3}.
$
Then the equilibrium $\Lambda = 0$ is fixed-time stable, and the settling time $T$ is upper-bounded by \textup{\cite{Zuo2015}}
\[
T \leq \frac{1}{\mathcal{L}_1^{\mathcal{B}_3} (1 - {\mathcal{B}_1}\mathcal{B}_3)} + \frac{1}{\mathcal{L}_2^{\mathcal{B}_3} ({\mathcal{B}_2\mathcal{B}_3} - 1)}.
\]
\label{lemma:fixed-time-stability}
\end{lemma}

\begin{assumption}
The combined image Jacobian matrix $\bm{L} \bm{J}_t$, which maps joint velocities to feature velocities, is assumed to be full row rank \textup{\cite{Siciliano2008}}.
\label{assum:invertibility}
\end{assumption}

\begin{assumption}
There exists a vector $\bm\lambda_1$ such that $\bm Q\bm L^\dagger \bm L\bm\lambda_1=1$ \textup{\cite{Ahmadi2022}}.
\label{assum:decoupler}
\end{assumption}

\begin{definition}  
The uncertainty terms 
$\bm{\Delta}_s$ and ${\Delta}_f$ in \eqref{eq:uncertain_error_system} are estimated by the following RBFNNs
\begin{equation}
\bm \Delta_s=\bm{W}_s^\top \bm\phi_s+\bm{\varepsilon}_s,\qquad 
\Delta_f=\bm{W}_f^\top \bm\phi_f+\varepsilon_f,
\label{eq:uncertain_parts}
\end{equation}

\noindent where $\bm{W}_s, \bm{W}_f$ are constant ideal weights, $\bm{\phi}_s,\bm{\phi}_f$ are prespecified basis vectors, and remainder terms $\bm{\varepsilon}_s,\varepsilon_f$ are bounded within $|\bm \varepsilon_s|\le \bar{\bm{\varepsilon}}_s$, $|\varepsilon_f|\le \bar\varepsilon_f$. 
The estimates are defined as $\hat {\bm \Delta}_s=\hat {\bm{W}}_s^\top \bm\phi_s$, $\hat\Delta_f=\hat {\bm{W}}_f^\top \bm\phi_f$, with their corresponding error signals $\tilde {\bm{W}}_s\triangleq \bm{W}_s-\hat {\bm{W}}_s$, $\tilde {\bm{W}}_f\triangleq \bm{W}_f-\hat {\bm{W}}_f$.
\end{definition}

\begin{definition} 
The shaper function $\Theta(\cdot)$ with constraint parameters $\mathcal{M}>1$, $\mathcal{N}\in(0,1)$, $\mathcal{J}>0$, $\mathcal{Q}>0$, and gains $\mathcal{C}_1, \mathcal{C}_2 > 0$ is defined {as follows} \textup{\cite{Gao2023}}:
\begin{equation}
\Theta(\bm{e}) = \left( \mathcal{C}_1 \|\bm{e}\|^{\frac{\mathcal{M}-1}{\mathcal{Q}\mathcal{J}}} + \mathcal{C}_2 \|\bm{e}\|^{\frac{\mathcal{N}-1}{\mathcal{Q}\mathcal{J}}} \right)^{\mathcal{Q}\mathcal{J}} \bm{e} + \operatorname{sig}^{\mathcal{J}}(\bm{e}).
\label{eq:theta_FNFTSM}
\end{equation}
\end{definition}

\begin{theorem}
Consider the image and force errors defined in \eqref{eq:error_main} and their corresponding time derivatives \eqref{eq:uncertain_error_system}. Suppose $\bm{\Gamma}_s$ and $\bm{\Gamma}_f$ are positive-definite diagonal gain matrices, $\delta_f$ and $\delta_s$ are positive scalar parameters, and $\lambda_s > 0$ is a weighting scalar. Under Assumptions \textup{\ref{assum:invertibility}} and \textup{\ref{assum:decoupler}}, applying the control law \eqref{eq:final_control_law} along with the adaptive laws \eqref{eq:adaptation_laws} guarantees that all closed-loop signals are bounded and that the error vector $(\bm{e}_s, e_f)$ converges to the origin in fixed time, uniformly with respect to the initial conditions:

\vspace{-0.3cm}
\begin{equation}
\begin{split}
\dot{\bm q}_{des} &= \bm{J}_t^\dagger \bm {L}^\dagger \bigg(
    \dot{\bm s}_d-\hat {\bm{\Delta}}_s+ \Theta(\bm {e}_s) \\
    &\quad + \hat{k}_e^{-1}\bm L \bm\lambda_1 \big(
      \dot f_{dn} -\hat\Delta_f+\Theta(e_f) \big) \\
    &\quad + k_s\,\operatorname{sat}\!\left(\frac{\bm e_s}{\delta_s}\right)
     + k_f\,\bm L\bm\lambda_1\,\operatorname{sat}\!\left(\frac{ e_f}{\delta_f}\right)
\bigg),
\end{split}
\label{eq:final_control_law}
\end{equation}
\noindent where $\hat{\Delta}_f$ and $\hat{\bm{\Delta}}_s$ in \eqref{eq:final_control_law} are the estimates of the uncertainty terms given in \eqref{feature_error_rate} and \eqref{eq:force_rate} when the following adaptive laws are employed:
\vspace{-0.2cm}

\begin{equation}
\begin{aligned}
\dot{\hat{\bm W}}_s &= \bm {\Gamma}_s \Big( \lambda_s \, \bm e_s \, \bm \phi_s^\top 
+ \hat{k}_e \, e_f \, \bm Q \, \bm L^\dagger \, \bm \phi_s^\top \Big),\\
\dot{\hat{\bm W}}_f &= \bm{\Gamma}_f \Big( \lambda_s \, \bm e_s \, \bm \phi_f^\top 
+ e_f \, \bm \phi_f^\top \Big).
\end{aligned}
\label{eq:adaptation_laws}
\end{equation}
In addition, the $\operatorname{sat}(\cdot)$ functions and their bounds are defined in \eqref{eq:sat_s} and \eqref{eq:sat_f}.
\end{theorem}

\begin{proof}
Consider the following Lyapunov candidate function:
\begin{equation}
V_t = \frac{1}{2} \lambda_s \|\bm e_s\|^2 + \frac{1}{2} e_f^2
+ \frac{1}{2} \tilde {\bm W}_s^\top \bm \Gamma_s^{-1} \tilde {\bm W}_s
+ \frac{1}{2} \tilde {\bm W}_f^\top \bm \Gamma_f^{-1} \tilde {\bm W}_f.
\label{eq:V_def}
\end{equation}

\noindent Differentiating $V_t$, using $\dot{\tilde {\bm W}}_s=-\dot{\hat {\bm W}}_s$, $\dot{\tilde {\bm W}}_f=-\dot{\hat {\bm W}}_f$, substitution of the control law \eqref{eq:final_control_law} and the adaptive relations \eqref{eq:adaptation_laws} lead to:

\begin{equation}
\label{eq:Vdot_with_R}
\begin{aligned}
\dot V_t
&= -\lambda_s\,\bm e_s^\top\Theta(\bm e_s) - e_f\,\Theta(e_f)\\
&\quad
- \lambda_s \bm e_s^\top \hat{k}_e^{-1} \bm L \bm\lambda_1 \,\Theta(e_f)
- e_f\, \hat{k}_e\,\bm Q \bm L^\dagger \,\Theta(\bm e_s)\\
&\quad
- \lambda_s\bm e_s^\top\bm\varepsilon_s - e_f\varepsilon_f\\
&\quad
-  \lambda_s\bm e_s^\top k_s \operatorname{sat}\!\Big(\frac{\bm e_s}{\delta_s}\Big)
- e_f\, k_f\, \operatorname{sat}\!\Big(\frac{e_f}{\delta_f}\Big).
\end{aligned}
\end{equation}

\noindent The first two terms of \eqref{eq:Vdot_with_R} are always negative semi-definite and can be used to dominate the cross-coupling terms. This is due to the following property:

\begin{align}
    e_f\Theta(e_f) &\ge \mathcal{C}_1^{\mathcal{Q}\mathcal{J}} |e_f|^{\mathcal{M}+1} + \mathcal{C}_2^{\mathcal{Q}\mathcal{J}}|e_f|^{\mathcal{N}+1} \nonumber \\
    &\quad + |e_f|^{\mathcal{J}+1} \ge \alpha_f|e_f|^{p_f}, 
    \label{eq:theta_lb2}
\end{align}
with
\begin{equation}
\begin{aligned}
    p_f &= \min\{\mathcal{M}+1,\; \mathcal{J}+1\} > 1, \\
    \alpha_f &= \min\{\mathcal{C}_1^{\mathcal{Q}\mathcal{J}},\; \mathcal{C}_2^{\mathcal{Q}\mathcal{J}},\; 1 \},
\end{aligned}
\end{equation}

\noindent representing the least dominant term coefficients. Using a similar approach for $\bm e_s$, there exist positive constants $\alpha_s,\alpha_f$ and exponents $p_s,p_f>1$ such that
\begin{equation}
\label{eq:theta_lb}
\bm e_s^\top\Theta(\bm e_s)\ge \alpha_s\|\bm e_s\|^{p_s},\qquad
e_f\Theta(e_f)\ge \alpha_f|e_f|^{p_f}.
\end{equation}

\noindent Let us define $
A\triangleq \hat{k}_e^{-1}\bm L\bm\lambda_1,\quad
B\triangleq \hat{k}_e\,\bm Q\bm L^\dagger,
$
with norms $\|A\|,\|B\|$, and $T_1=\lambda_s \bm e_s^\top \hat{k}_e^{-1} \bm L \bm\lambda_1 \,\Theta(e_f)$, $T_2=e_f\, \hat{k}_e\,\bm Q \bm L^\dagger \,\Theta(\bm e_s)$, therefore:
\[
|T_1|\le \lambda_s\|A\|\,\|\bm e_s\|\,|\Theta(e_f)|,\quad
|T_2|\le \|B\|\,|e_f|\,\|\Theta(\bm e_s)\|.
\]
Since $\|\Theta(\bm e_s)\|\!\le\! C_{\Theta s}\|\bm e_s\|^{q_s}$, $|\Theta(e_f)|\!\le\! C_{\Theta f}|e_f|^{q_f}$, Young’s inequality with $\eta_{1,2}\in(0,1)$ yields
\begin{equation}
\label{eq:T1T2_bound}
|T_1|+|T_2| \le \tfrac14\lambda_s\alpha_s\|\bm e_s\|^{p_s} + \tfrac14\alpha_f|e_f|^{p_f}
+ c_1\|\bm e_s\|^{\beta_s} + c_2|e_f|^{\beta_f},
\end{equation}
with $\beta_s,\beta_f\!\ge\!1$ and $c_{1,2}\!\ge\!0$ depending on $\|A\|,\|B\|,C_{\Theta s},C_{\Theta f},\eta_{1,2},\alpha_{s,f}$. Choosing exponents ensures $\beta_s<p_s,\;\beta_f<p_f$ or makes $c_{1,2}$ arbitrarily small. For each component $e_{s,i}$ of $\bm e_s$, define the standard saturation function as
\begin{equation}
\operatorname{sat}\!\left(\frac{e_{s,i}}{\delta_s}\right)
=
\begin{cases}
\operatorname{sign}(e_{s,i}), & |e_{s,i}| \ge \delta_s,\\[2mm]
\dfrac{e_{s,i}}{\delta_s}, & |e_{s,i}| < \delta_s.
\end{cases}
\label{eq:sat_s}
\end{equation}
Likewise,
\begin{equation}
\operatorname{sat}\!\left(\frac{e_f}{\delta_f}\right)
=
\begin{cases}
\operatorname{sign}(e_f), & |e_f| \ge \delta_f,\\[2mm]
\dfrac{e_f}{\delta_f}, & |e_f| < \delta_f.
\end{cases}
\label{eq:sat_f}
\end{equation}

\noindent Let\[
\mathcal R
=
\lambda_s k_s \bm e_s^\top
\operatorname{sat}\!\left(\frac{\bm e_s}{\delta_s}\right)
+
k_f e_f
\operatorname{sat}\!\left(\frac{e_f}{\delta_f}\right).
\]
Therefore,
\begin{equation}
\label{eq:R_lower}
\mathcal R
\ge
\lambda_s k_s
\min\!\left\{
2N_l \delta_s,\;
\frac{\|\bm e_s\|^2}{\delta_s}
\right\}
+
k_f
\min\!\left\{
\delta_f,\;
\frac{e_f^2}{\delta_f}
\right\}.
\end{equation}

\noindent The remainder contributions satisfy 
$
\lambda_s|\bm e_s^\top\bm\varepsilon_s| \le \lambda_s\|\bm e_s\|\,\bar{\varepsilon}_s,\qquad
|e_f\varepsilon_f| \le |e_f|\,\bar{\varepsilon}_f.
$
We now use the robust term $\mathcal R$ to dominate these linear-in-error remainders. If $|e_{s,i}|<\delta_s$ for all $i$ and $|e_f|<\delta_f$, then \eqref{eq:R_lower} implies the robust term provides quadratic terms
\[
\mathcal R \ge \lambda_s k_s \frac{\|\bm e_s\|^2}{\delta_s} + k_f \frac{e_f^2}{\delta_f}.
\]
The Young's inequality is used to show these quadratic terms dominate the linear-in-error remainder bounds $\lambda_s\|\bm e_s\|\,\bar{\varepsilon}_s$ and $|e_f|\,\bar{\varepsilon}_f$:
\vspace{-0.2cm}
\[
\lambda_s\|\bm e_s\|\,\bar{\varepsilon}_s
\le \tfrac12 \lambda_s k_s \frac{\|\bm e_s\|^2}{\delta_s} + \frac{\lambda_s\bar{\varepsilon}_s^2\delta_s}{2\lambda_s k_s},
\]
and similarly for the force channel. It is feasible to select $k_s,k_f,\delta_s,\delta_f$ such that quadratic robust expressions dominate the linear remainders with a margin. {The higher the ratio $k_{s,f}/\delta_{s,f}$, the shorter the  convergence time but the larger the steady-state bound}. The bounds \eqref{eq:T1T2_bound}--\eqref{eq:R_lower} {and the remainder cancellations can be substituted into} \eqref{eq:Vdot_with_R}. By the {above construction} there exist choices of $k_s,k_f,\delta_s,\delta_f$ and the small Young parameters $\eta_{1,2}$ such that the sum of the cross-coupling terms and the remainders is fully compensated by the robust term, i.e., there exist positive constants $\bar\kappa_s,\bar\kappa_f>0$ and exponents $\mu_s,\mu_f\ge 1$ with
\[
\dot V_t \le -\bar\kappa_s \|\bm e_s\|^{\mu_s} - \bar\kappa_f |e_f|^{\mu_f},\qquad
\forall (\bm e_s,e_f)\neq(\bm 0,0).
\]

\noindent Since $\|\bm e_s\| \le \sqrt{2V_t}$ and $|e_f| \le \sqrt{2V_t}$, we have:
\[
\dot V_t 
\le
-\bar\kappa_s (\sqrt{2}V_t)^{\mu_s/2}
-\bar\kappa_f (\sqrt{2}V_t)^{\mu_f/2}.
\]
Defining $\mathcal{B}_1=\bar\kappa_s 2^{-\mu_s/2}$ and $\mathcal{B}_2=\bar\kappa_f 2^{-\mu_f/2}$, one obtains
\[
\dot V_t \le - (\mathcal{B}_1 V_t^{\mu_s/2} + \mathcal{B}_2 V_t^{\mu_f/2}).
\]
For $0<\mu_s<2$ and $\mu_f>2$, the closed-loop system is fixed-time stable according to Lemma~\ref{lemma:fixed-time-stability}.

\end{proof}


\vspace{-1cm}
\subsection{Projection Compensation Employing Virtual Camera}
To address the underactuation of multirotor UAVs, it is not feasible to directly control the roll and pitch joint variables. To overcome this limitation in visual servoing, a method similar to \cite{Kim2016, Xu2023,Xiu2025} is adopted. This method constructs visual features within a virtual camera frame whose origin coincides with that of the actual camera frame, while its roll and pitch angles are constrained to zero and its yaw angle is identical to that of the aerial platform. As a result, the controllable joint variables associated with the roll and pitch of the aerial platform are effectively eliminated.

\section{Simulations and Experiments}\label{s4}

This section presents several simulation and experimental studies conducted to verify and analyze the performance of the proposed controller.

The TD-ACM parameters used in the simulations are established as follows: the backbone density is $6450~[\mathrm{kg/m^3}]$, with a Young's modulus of $50 \times 10^9~[\mathrm{Pa}]$ and a material damping coefficient of $10^9~[\mathrm{Pa\cdot s}]$. Poisson's ratio is set to $0.33$. The UAV mass is $2.3562~[\mathrm{kg}]$. Geometrically, the arm is defined by a backbone radius of $1~[\mathrm{mm}]$ and a tendon-to-centerline distance of $2~[\mathrm{cm}]$. 

Fig.~\ref{fig:exp_setup} illustrates the annotated indoor experimental setup. The hardware configuration includes a Nano17 six-axis F/T sensor (ATI Industrial Automation, NC, USA) for measuring interaction forces and a Logitech C615 HD camera (Logitech, CA, USA) with a set resolution of $640 \times480\ [\mathrm{pixel}]$ and $30 \ [\mathrm{fps}]$, both mounted at the tip via a custom-fabricated PLA holder. A Vicon motion capture system (Vicon Motion Systems Ltd., UK) is employed to estimate the tip pose in the reference frame. The CR is composed of several elements including backbone, spacer disks, tendons, actuation units, and an aluminum base. The backbone is made of Nitinol with a density of $7800\ \mathrm{[kg/m^3]}$,
Young’s modulus $207\ \mathrm{[GPa]}$, length $0.4\ \mathrm{[m]}$, and radius $0.9\ \mathrm{[mm]}$. Four braided fishing lines serve as tendons routed along a circle with a radius of $18.55\ \mathrm{[mm]}$, running parallel to the backbone.

\begin{figure}
\vspace{0.4cm}
    \centering
    \includegraphics[scale = 0.4]{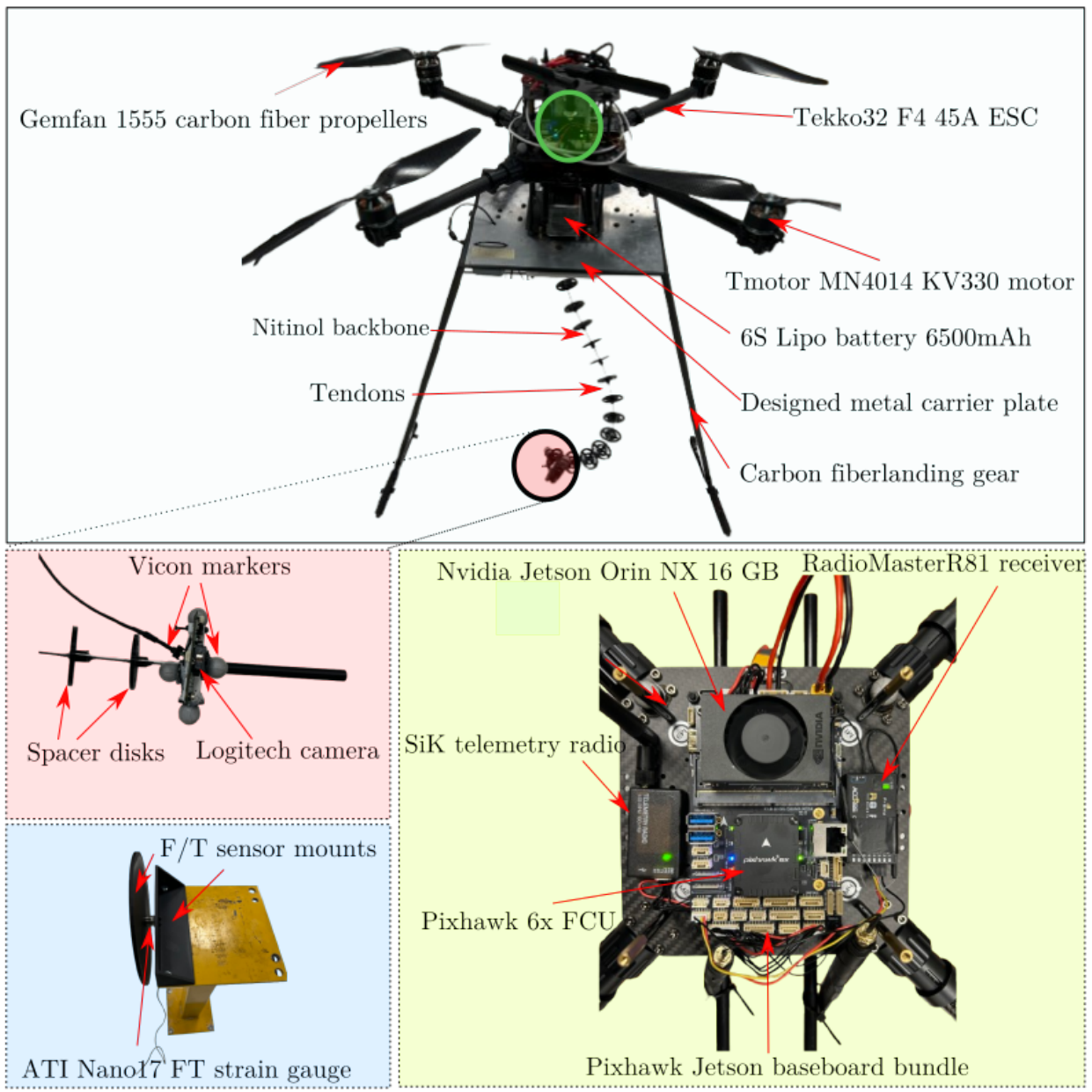}
    \caption{The components of the experimental setup.}
    \label{fig:exp_setup}
\end{figure}


\subsection{Simulation Studies}

The simulation studies comprise three parts: quasi-static tracking and robustness to initial conditions; time-varying vision and force signal tracking demonstrating error convergence; and a comparative controller performance analysis.

In the first test, a square trajectory is adopted for the desired image profile $
\bm{s}_d = \begin{bmatrix} \bm{s}_{d_{\theta}}^\top & \bm{s}_{d_{\rho}}^\top \end{bmatrix}^\top$, with \linebreak $\bm{s}_{d_{\theta}}= \begin{bmatrix} -0.8571 & 3.1416 & -2.2845 
\end{bmatrix}\ \mathrm{[rad]}$ and $\bm{s}_{d_{\rho}}$ is a piecewise constant function. The desired quasi-static normal interaction force profile is 
$f_{dn} = 2.5~\mathrm{[N]}$ for $0 < t \leq 17 ~\mathrm{[s]}$, 
and $0.5~\mathrm{[N]}$ otherwise. Fig.~\ref{fig:TestB}(a) to Fig.~\ref{fig:TestB}(d) illustrate the simulation results for two distinct initial conditions. The total elapsed time of this test is $25\ \mathrm{[s]}$ with a step-size of $1\ \mathrm{[ms]}$. The estimated stiffness of the environment is considered as $\hat{k}_e=10^4\ \mathrm{[N/m]}$. The target is an equilateral triangle with a side of $0.2\ \mathrm{[m]}$ lying in the XZ plane of the inertial frame. The approximate plane parameters are $a_p = 0$, $b_p =0$, $c_p = 3$, and $d_p = -1$. The line features including $\theta_i$ and $\rho_i$ for $i \in \{1,2,3\}$ are depicted in Fig.~\ref{fig:TestB}(a) and \ref{fig:TestB}(b), for both of the initial conditions. The fluctuations between $\pm\pi\ \mathrm{[rad]}$ stem from the wrapped difference operator on the angles. The normal forces for both initial conditions are shown for the entire simulation time in Fig.~\ref{fig:TestB}(c). The free-flight mode lasts less than $0.05\ \mathrm{[s]}$, and the system exhibits transient oscillations following a peak contact force of approximately $12\ \mathrm{[N]}$. During the first second of the run, the interaction force approaches the first desired set-point of $2.5\ \mathrm{[N]}$, followed by regulation of the force to $0.5\ \mathrm{[N]}$ after $t=17\ \mathrm{[s]}$. Fig.~\ref{fig:TestB}(d) exhibits the projected image of the target for the two different initial poses (dashed lines for the first case and solid lines for the second), where the target image (magenta triangle) and vertex traces are superimposed. Lastly, the 3D visualization of the TD-ACM, together with the tip trace, workpiece, and features, is illustrated in Fig.~\ref{fig:TestBDemo}.

\begin{figure}[t] 
\centering
\begin{subfigure}[b]{0.48\linewidth}
    \centering
    \includegraphics[width=\linewidth]{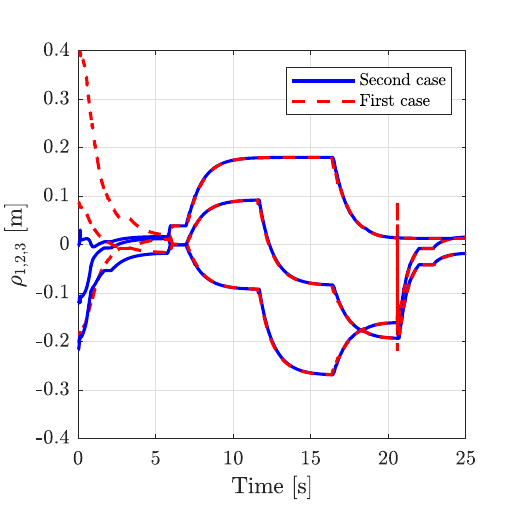}
    \caption{}
    \label{fig:square_6a}
\end{subfigure}
\hfill
\begin{subfigure}[b]{0.48\linewidth}
    \centering
    \includegraphics[width=\linewidth]{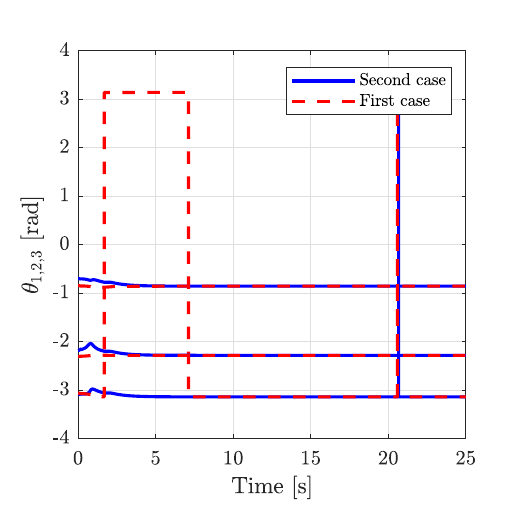}
    \caption{}
    \label{fig:square_6b}
\end{subfigure}

\vspace{2mm} 

\begin{subfigure}[b]{0.48\linewidth}
    \centering
    \includegraphics[width=\linewidth]{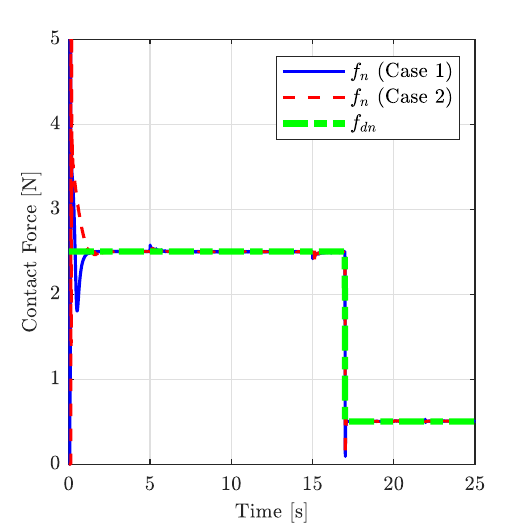}
    \caption{}
    \label{fig:square_6c}
\end{subfigure}
\hfill
\begin{subfigure}[b]{0.48\linewidth}
    \centering
    \includegraphics[width=\linewidth]{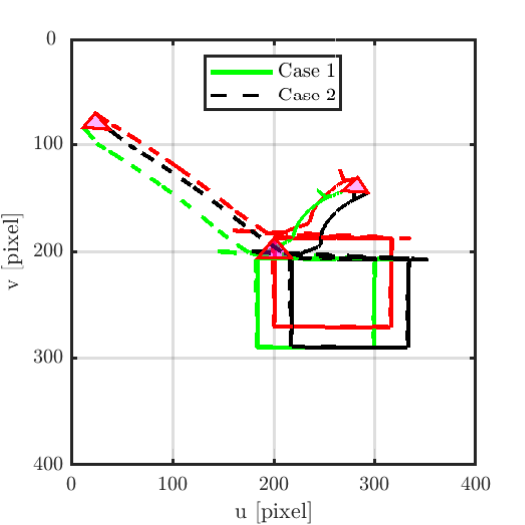} 
    \caption{}
    \label{fig:square_6d}
\end{subfigure}

\caption{Simulation results: (a) and (b) line feature evolution; (c) force tracking comparison; and (d) initial/final images and vertex trace.}
\label{fig:TestB}
\end{figure}

\begin{figure}
\centering
\includegraphics[scale = 0.18]{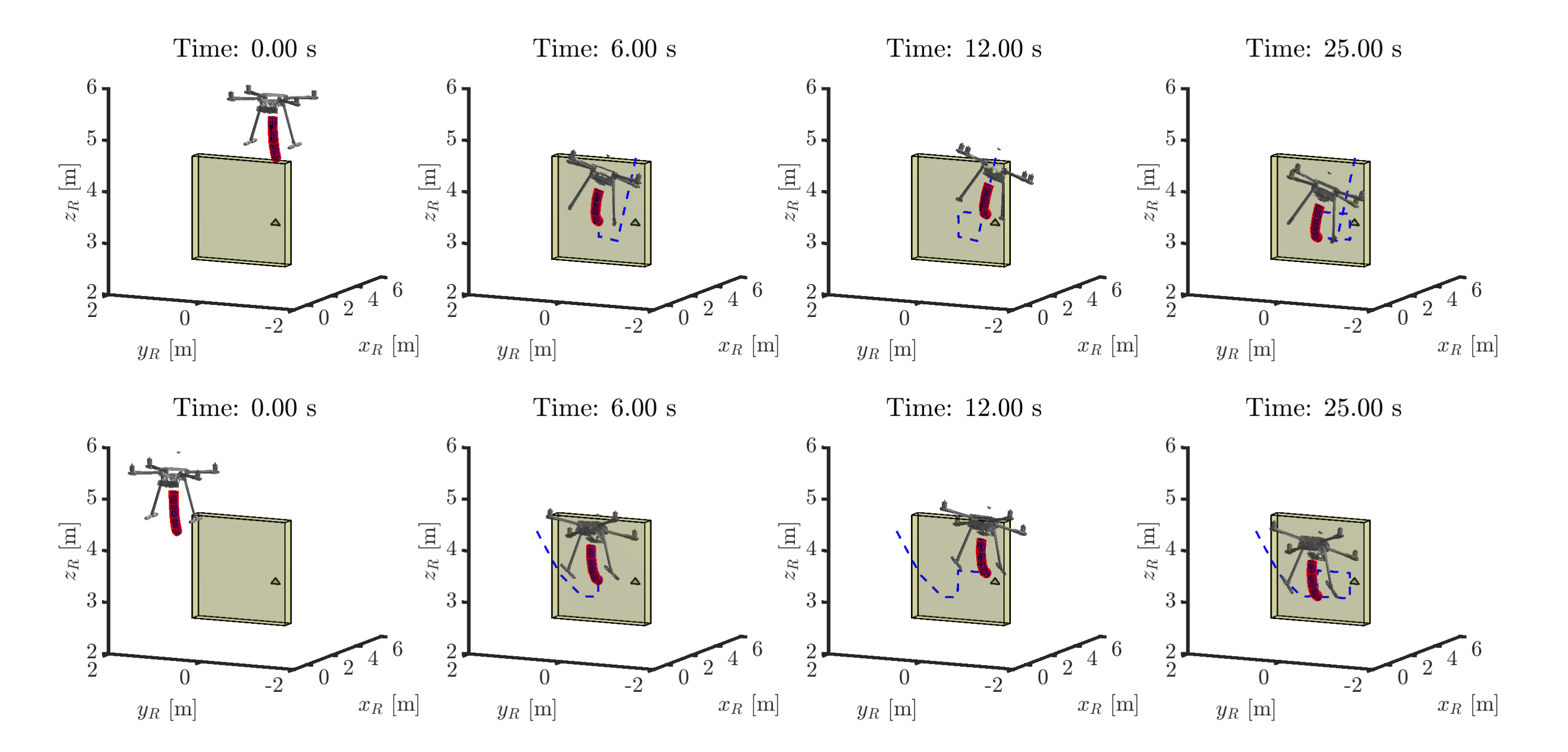}
\caption{Sequential snapshots at various time frames for Case 1 (top) and Case 2 (bottom). The dashed blue line indicates the CR tip trace.}
\label{fig:TestBDemo}
\end{figure}

In the second test, the desired $\bm{s}_d$ is captured when the camera traverses an infinity-shaped trajectory, while the desired force profile is $ f_{dn}=2.5(1+\sin{\frac{12.5t}{2\pi}})\ [\mathrm{N}]$. The results of this test are depicted in Fig.~\ref{fig:TestC}(a) to Fig.~\ref{fig:TestC}(d), and the graphical results, as shown in Fig.~\ref{fig:TestCDemo}, demonstrate the corresponding 3D visualization at various time instances. Fig.~\ref{fig:TestC}(a) and Fig.~\ref{fig:TestC}(b) illustrate the image features for both linear and angular components, respectively. The image feature error vector stabilizes, as shown in Fig.~\ref{fig:TestC}(c). The initial and the final images of the target are also indicated in Fig.~\ref{fig:TestC}(d). Fig.~\ref{fig:TestC}(e) shows the force setpoint, measured force, and normal tip displacement $p_t$. Stable contact is achieved within milliseconds and maintained for the duration of the process. Lastly, Fig.~\ref{fig:TestCDemo} illustrates the infinity-shaped tracking while maintaining the contact with the surface during the maneuver for several time instances over an extended simulation time of $25\ \mathrm{[s]}$. 
\begin{figure}[t] 
    \centering
    \setlength{\tabcolsep}{1pt} 

    \begin{subfigure}[b]{0.48\linewidth}
        \centering
        \includegraphics[width=\linewidth]{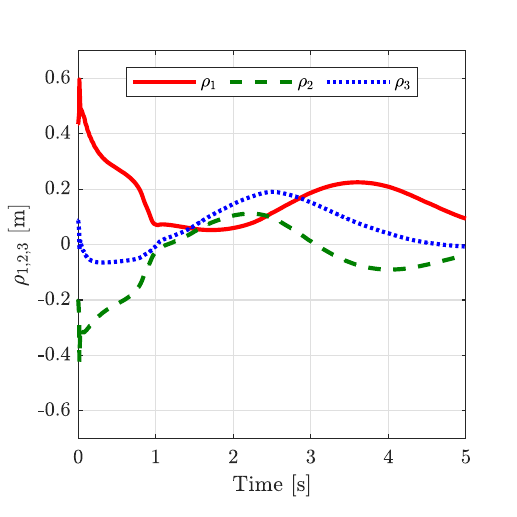}
        \caption{}
        \label{fig:inf_9a}
    \end{subfigure}
    \hfill
    \begin{subfigure}[b]{0.48\linewidth}
        \centering
        \includegraphics[width=\linewidth]{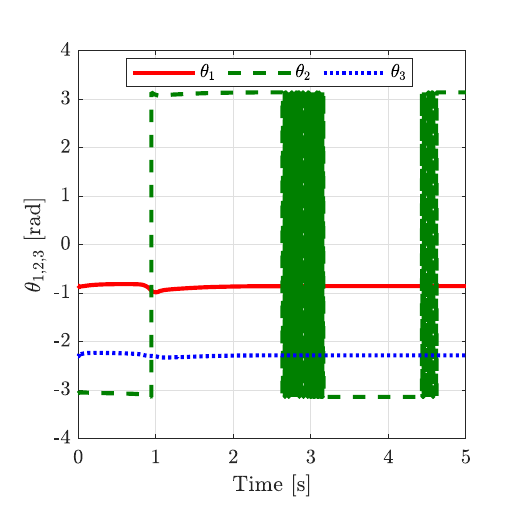}
        \caption{}
        \label{fig:inf_9b}
    \end{subfigure}

    \vspace{2mm} 

    \begin{subfigure}[b]{0.48\linewidth}
        \centering
        \includegraphics[width=\linewidth]{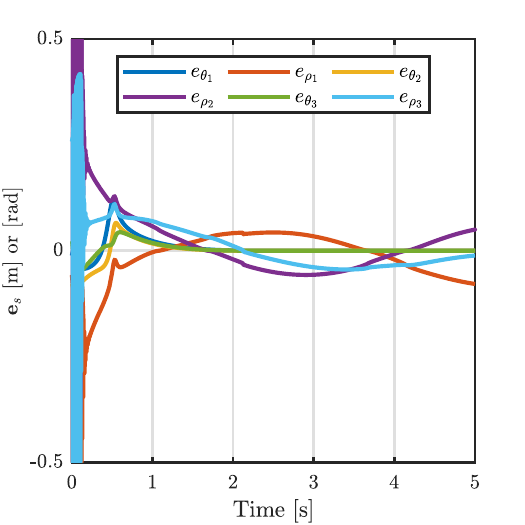}
        \caption{}
        \label{fig:inf_8d}
    \end{subfigure}
    \hfill
    \begin{subfigure}[b]{0.48\linewidth}
        \centering
        \includegraphics[width=\linewidth]{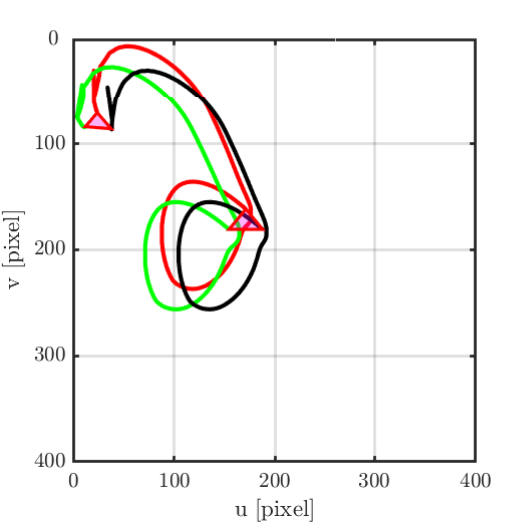}
        \caption{}
    \end{subfigure}

    \begin{subfigure}[b]{\linewidth}
        \centering
        \includegraphics[width=0.48\linewidth]{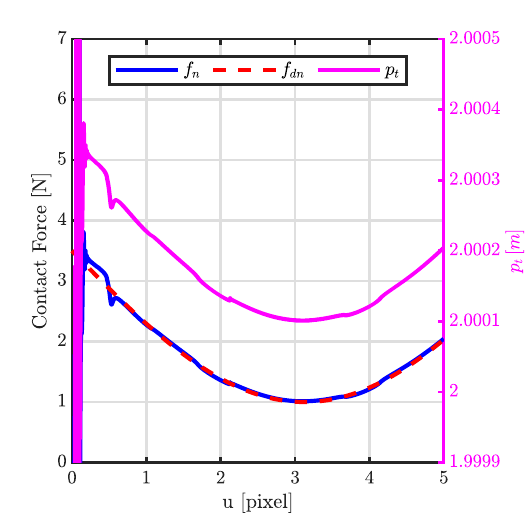}
        \caption{}
    \end{subfigure}

    \caption{Simulated responses: (a) and (b) line features; (c) image feature error; (d) image space; (e) interaction force tracking and displacement.}
    \label{fig:TestC}
\end{figure}

In the third test, we investigate the performance of the proposed scheme as opposed to two other methods traditionally applied in rigid robotics, the PI/PD and CISMC. To the best of our knowledge, such comparative studies have not been previously conducted for the TD-ACM. The reference trajectories introduced in the second test are selected in a timeframe of $5\ \mathrm{[s]}$. Table~\ref{tab:comparison} lists the performance of the proposed method in terms of RMSE, STD of error, integral absolute error (IAE), and integral time absolute error (ITAE). 

The results show improved performance in both visual and force tracking compared to other methods. Regarding vision error, the proposed method achieves an RMSE of $0.06$, almost half that of the PD controller ($0.12$), and improves the STD by $28\%$. However, an increase is visible in the IAE and ITAE metrics, indicating more fluctuations in the transient response. The proposed method outperforms the CISMC by more than $50\%$ in all criteria. Likewise, force tracking is improved by $40\%$ in terms of STD and RMSE compared to the CISMC, while the enhancement in ITAE and IAE is $9\%$ and $43\%$, respectively.

\begin{figure}[b]
    \centering
    \includegraphics[scale = 0.18]{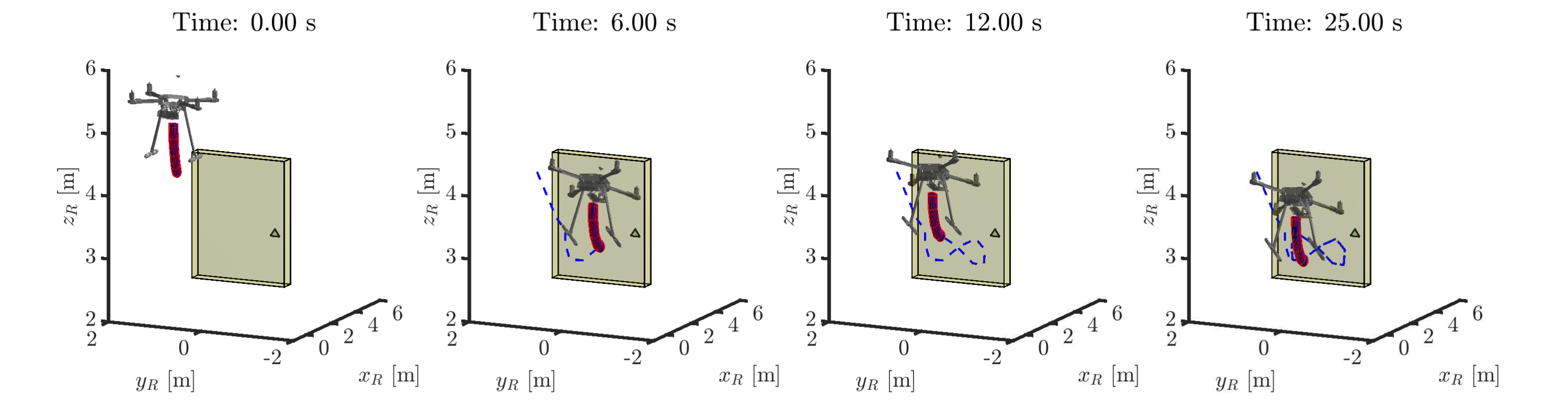}
\caption{3D visualization of the TD-ACM. The dashed blue line indicates the CR tip trace.}
    \label{fig:TestCDemo}
\end{figure}

\begin{figure}[t!]
    \centering

 \centering
    \begin{subfigure}{0.48\linewidth} 
        \centering
        \includegraphics[width=\linewidth]{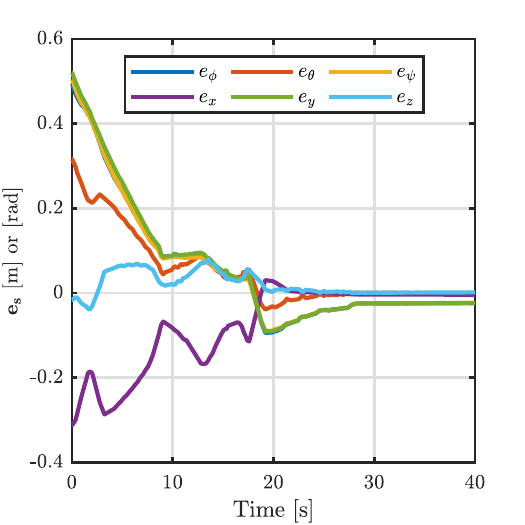}
        \caption{}
    \end{subfigure}\hfill
    \begin{subfigure}{0.48\linewidth}
        \centering
        \includegraphics[width=\linewidth]{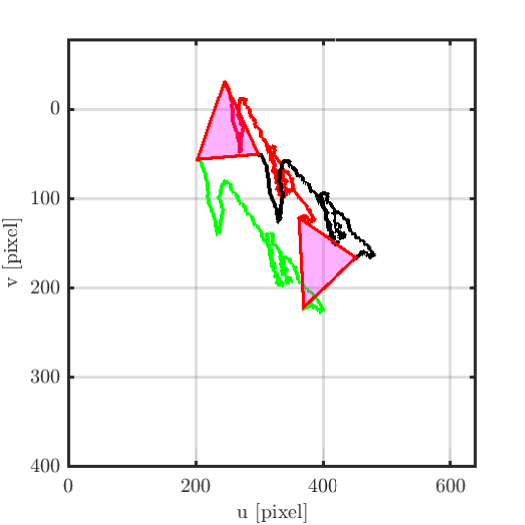}
        \caption{}
    \end{subfigure}\hfill
    
    \begin{subfigure}{0.48\linewidth}
        \centering
        \includegraphics[width=\linewidth]{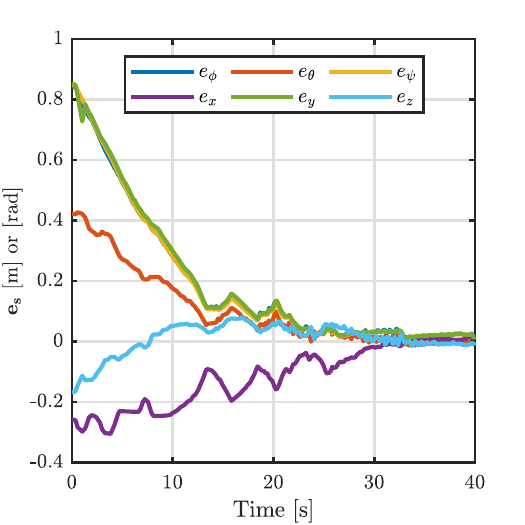}
        \caption{}
    \end{subfigure}\hfill
    \begin{subfigure}{0.48\linewidth}
        \centering
        \includegraphics[width=\linewidth]{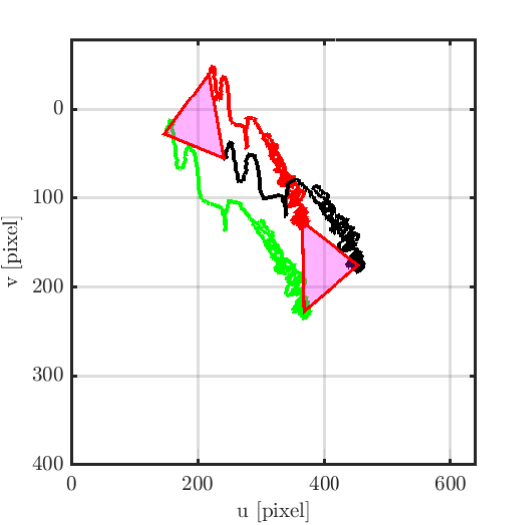}
        \caption{}
    \end{subfigure}\hfill

\caption{Contactless experimental tests: (a, b) the first initial condition; (c, d) the second initial condition.}
\label{fig:ContactLess_Exp_First}
\end{figure}

\begin{table}[ht]
\centering
\caption{The computed indices for vision/force control in Test 2.}
\footnotesize 
\setlength{\tabcolsep}{4pt} 
\begin{tabular}{llrrrr}
\toprule
\textbf{Approach} & & \textbf{RMSE} & \textbf{STD} & \textbf{ITAE} & \textbf{IAE} \\
\midrule
\multirow{3}{*}{Vision} 
& PD & 0.12 & 0.07 & 3.43 & 1.67 \\
& ISMC & 0.67 & 0.66 & 35.32 & 68.26 \\
& Proposed & 0.06 & 0.05 & 16.48 & 11.60 \\
\midrule
\multirow{3}{*}{Force} 
& PI & 4.49 & 4.50 & 10.87 & 16.09 \\
& ISMC & 2.34 & 2.33 & 27.35 & 131.6 \\
& Proposed & 1.39 & 1.39 & 24.77 & 74.82 \\
\bottomrule
\end{tabular}
\label{tab:comparison}
\end{table}
\subsection{Experiment - Contactless Ground Tests}

This experiment consists of two parts, for the sake of comparison. The first experiment evaluates controller performance using AprilTag markers, while the second utilizes deep feature extraction. The image-space graphs for two different initial conditions under the same control parameters are exhibited in Fig.~\ref{fig:ContactLess_Exp_First}. The line features coincide with the edges of an equilateral triangle defined by the marker centers. As shown in Fig.~\ref{fig:ContactLess_Exp_First}(a) and Fig.~\ref{fig:ContactLess_Exp_First}(c), the controller's stability is demonstrated by the convergence of the error $\bm{e}_s$ to a negligible value within approximately $25\ \mathrm{[s]}$ in both cases. The image-space evolution and the resulting feature trajectories are depicted in Fig.~\ref{fig:ContactLess_Exp_First}(b) and Fig.~\ref{fig:ContactLess_Exp_First}(d). It can be observed that despite differences in the transient trajectories, the final projected features reach nearly identical steady-state configurations.


The second part of this experiment evaluates the performance of the proposed control law when integrated with the deep feature extractor. Unlike in conventional methods, the number of detected line features varies dynamically across frames. As illustrated in Fig.~\ref{fig:contactfree_exp_part1}(a), the Euclidean norm of the image feature error vector converges to less than $0.01$ within $20\ [s]$. Furthermore, the commanded twist vector shown in Fig.~\ref{fig:contactfree_exp_part1}(b) remains consistently below $1\ \mathrm{[m/s]}$ and approaches zero after $20\ \mathrm{[s]}$, demonstrating the stability of the vision-based control system. The global tip pose, including camera position and Euler angles, was recorded using a Vicon motion capture system and is presented in Fig.~\ref{fig:contactfree_exp_part1}(c) and Fig.~\ref{fig:contactfree_exp_part1}(d), respectively. Observed variations in the tip position remain within the range of $2$--$2.5\ \mathrm{[cm]}$. 

\begin{figure}[t]
    \centering
    \setlength{\tabcolsep}{1pt}
    \begin{subfigure}{0.48\linewidth} 
        \centering
        \includegraphics[width=\linewidth]{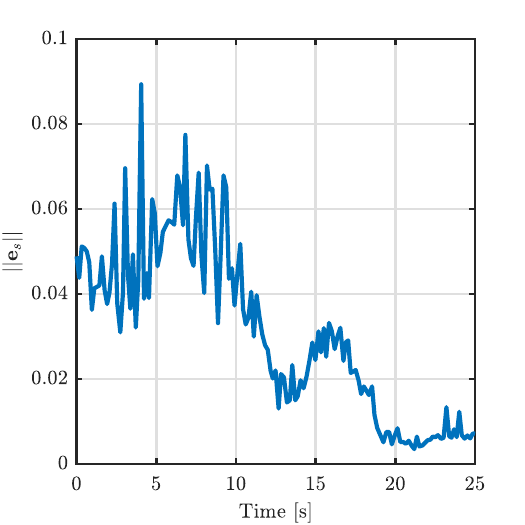}
        \caption{}
    \end{subfigure}
    \hfill
    \begin{subfigure}{0.48\linewidth}
        \centering
        \includegraphics[width=\linewidth]{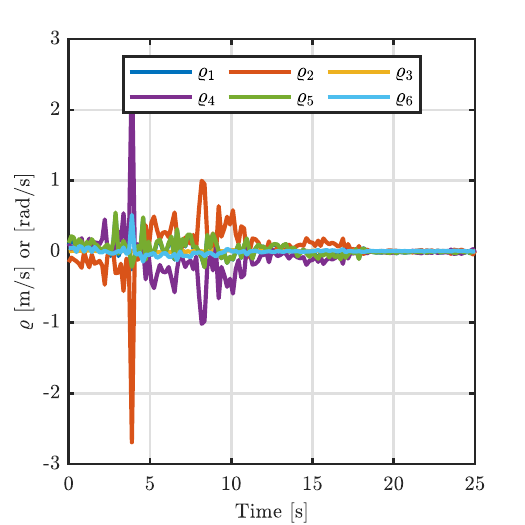}
        \caption{}
    \end{subfigure}
    
    \begin{subfigure}{0.48\linewidth}
        \centering
        \includegraphics[width=\linewidth]{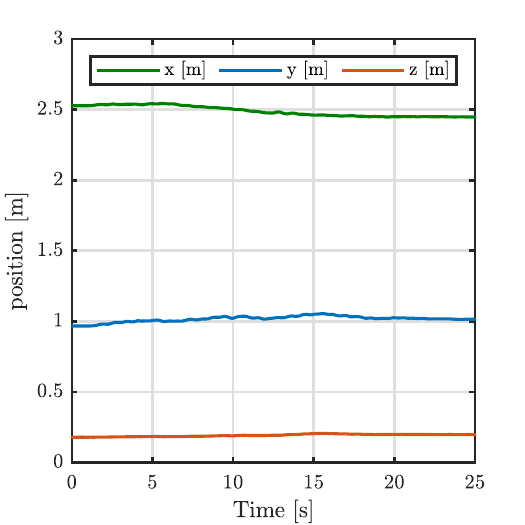}
        \caption{}
    \end{subfigure}
    \hfill
    \begin{subfigure}{0.48\linewidth}
        \centering
        \includegraphics[width=\linewidth]{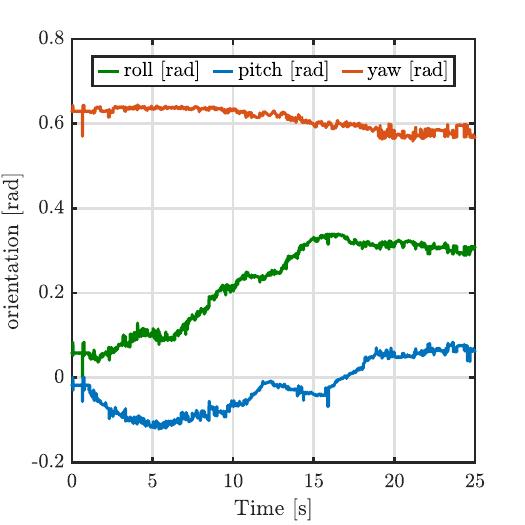}
        \caption{}
    \end{subfigure}

\caption{Contactless experiment results for the second scenario: (a) image error norm; (b) camera twist; (c) camera translation; and (d) camera orientation.}
    \label{fig:contactfree_exp_part1}
\end{figure}

\vspace{-0.3cm}
\subsection{Experiment - Push Test With Various Scenes}

This experiment characterizes the performance of the hybrid vision–force controller based on the proposed control law. A secondary objective is to assess the impact of environmental variations by modifying the target scene across two experimental cases. 

Fig.~\ref{fig:Contact_Exp}(a) shows the Euclidean norm of the image feature error alongside the normal force error over a duration of $120\ \mathrm{[s]}$, with a reference force set to $f_{dn} = 0.1\ \mathrm{[N]}$. The CR exploits the inherent compliance of the Nitinol backbone to safely regulate small forces, enabling delicate interactions such as those with deformable objects. While contact is established within the first $10\ \mathrm{[s]}$, a transient overshoot in the vision loop causes a momentary loss of contact before it is successfully re-established. Both vision and force errors converge to a steady state after approximately $50\ \mathrm{[s]}$. The corresponding camera twist is illustrated in Fig.~\ref{fig:Contact_Exp}(b) for the first scene. The interaction wrench exerted by the workpiece on the F/T sensor is plotted in Fig.~\ref{fig:Wrench_EXP}. During contact, the tangential force component remains below $0.02\ \mathrm{[N]}$, while the first component of torque reaches just above $2\ \mathrm{[Nmm]}$.

The combined image feature norm and normal force errors for the second scene are illustrated in Fig.~\ref{fig:Contact_Exp}(c). The second case contains a higher density of linear features; consequently, despite a larger initial feature error, this configuration achieves faster convergence (under $10~\mathrm{[s]}$) and requires smaller commanded twist values compared to the first case, as shown in Fig.~\ref{fig:Contact_Exp}(d). The resulting camera views for both scenarios are provided in Fig.~\ref{fig:Contact_Exp_camview}. The second scenario exhibits faster error stabilization because the increased number of features leads to enhanced conditioning of the interaction matrix.

\begin{figure}[t]
    \centering

 \centering
    \begin{subfigure}{0.48\linewidth} 
        \centering
        \includegraphics[width=\linewidth]{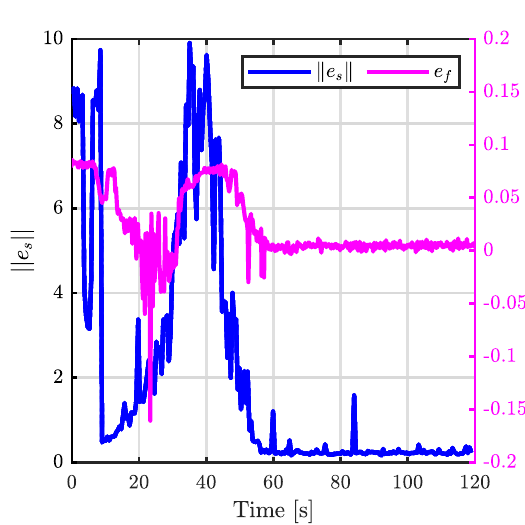}
        \caption{}
    \end{subfigure}\hfill
    \begin{subfigure}{0.48\linewidth}
        \centering
        \includegraphics[width=\linewidth]{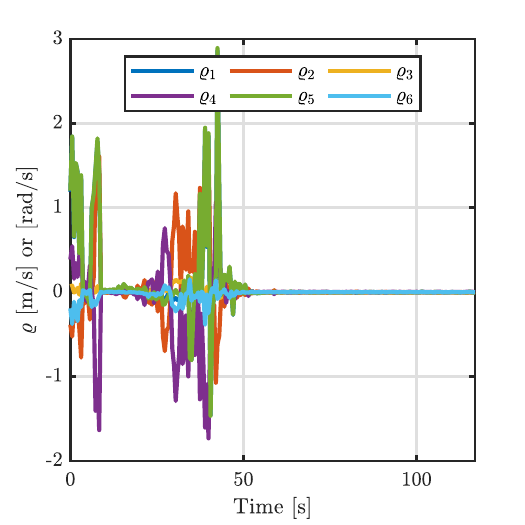}
        \caption{}
    \end{subfigure}\hfill
    
    \begin{subfigure}{0.48\linewidth}
        \centering
        \includegraphics[width=\linewidth]{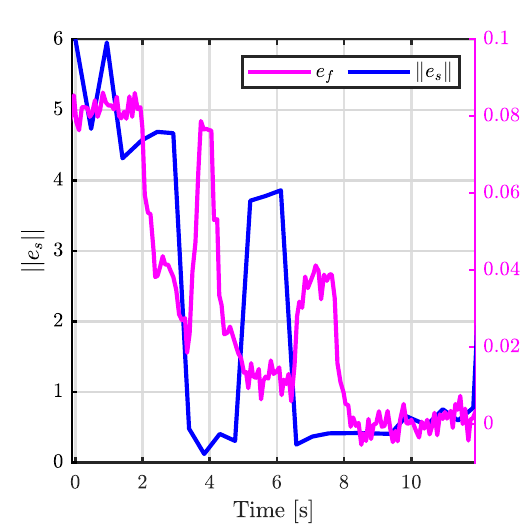}
        \caption{}
    \end{subfigure}\hfill
    \begin{subfigure}{0.48\linewidth}
        \centering
        \includegraphics[width=\linewidth]{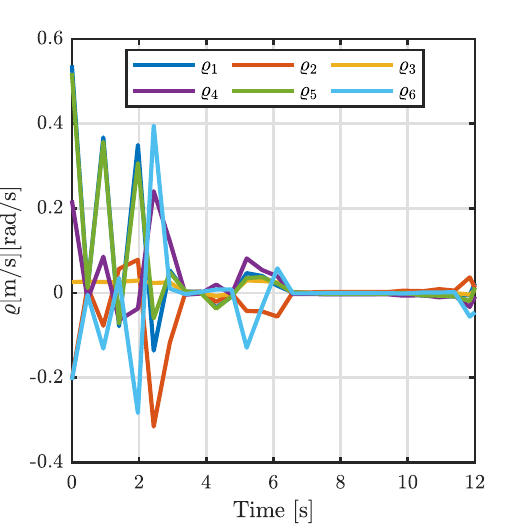}
        \caption{}
    \end{subfigure}\hfill

    \begin{subfigure}{0.48\linewidth}
        \centering
        \includegraphics[width=\linewidth]{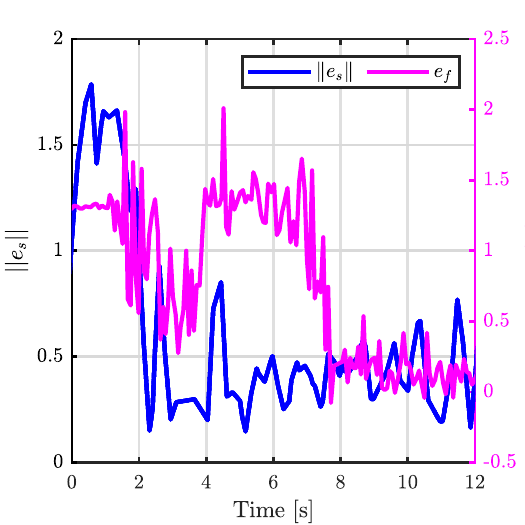}
        \caption{}
    \end{subfigure}\hfill
    \begin{subfigure}{0.48\linewidth}
        \centering
        \includegraphics[width=\linewidth]{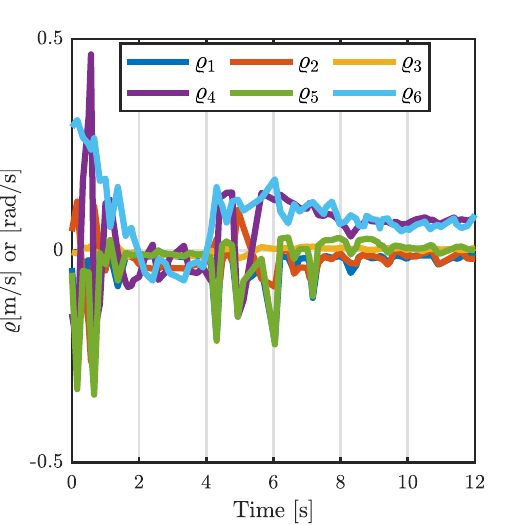}
        \caption{}
    \end{subfigure}\hfill
    
\caption{Contact experiment results: (a, b) first ground test scenario; (c, d) second ground test scenario; and (e, f) aerial test.}

\label{fig:Contact_Exp}
\end{figure}

\begin{figure}[t]
\centering
    \includegraphics[width=0.48\linewidth]{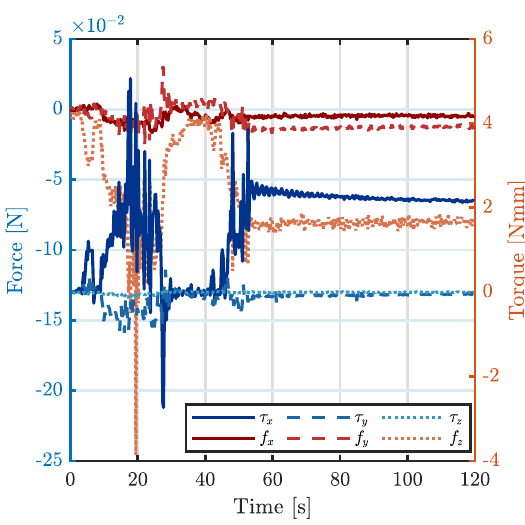}
    \caption{The interaction wrench for the first contact experimental test.}
\label{fig:Wrench_EXP}
\end{figure}

\begin{figure}[ht]
    \centering

 \centering
    \begin{subfigure}{\linewidth} 
        \centering
        \includegraphics[width=\linewidth]{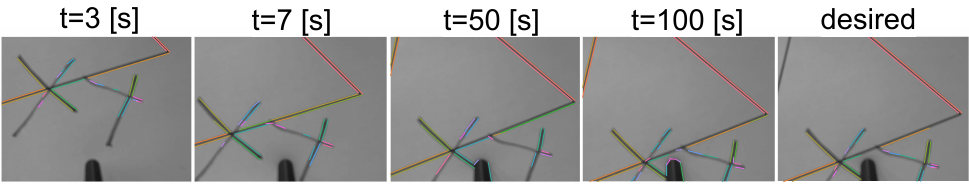}
        \caption{}
    \end{subfigure}\hfill
    
    \begin{subfigure}{\linewidth}
        \centering
        \includegraphics[width=\linewidth]{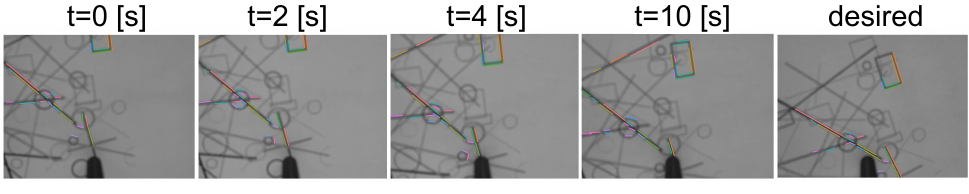}
        \caption{}
    \end{subfigure}\hfill

\caption{Camera views from the ground contact tests: (a) first scenario; (b) second scenario.}

    \label{fig:Contact_Exp_camview}
\end{figure}

\subsection{Experiment - Push Test in Aerial Scenario}

This study validates the efficacy of the proposed hybrid vision/force controller for aerial applications, with results illustrated in Fig.~\ref{fig:Contact_Exp}(e) and Fig.~\ref{fig:Contact_Exp}(f). Fig.~\ref{fig:Contact_Exp}(e) depicts the evolution of the image feature error norm $\|\bm{e}_s\|$ and the normal force error $e_f$ over a $12\ \mathrm{[s]}$ duration, using a reference force of $f_{dn} = 1.3\ \mathrm{[N]}$. The visual error converges in less than $5\ \mathrm{[s]}$, while the force signal reaches a steady state within $7\ \mathrm{[s]}$. The commanded local camera twist components are shown in Fig.~\ref{fig:Contact_Exp}(f). In contrast to the tests conducted on the ground, this test exhibits increased oscillations in the error signals due to aerodynamic disturbances, specifically the UAV downwash effect on the CR.


\section{Conclusions}\label{s5}
This work reports the design of a hybrid motion/force controller for a coupled TD-ACM based on constant curvature modeling. The novel controller combines cascaded fixed-time vision/force loops using line features and estimates uncertainties using an RBFNN. The line features are extracted using a GNN framework that leverages connectivity information and is shared with the controller to execute line-based IBVS, maintaining simultaneous motion and force tracking on a flat vertical surface. The fixed-time stability proof, underlying assumptions, and definitions of the proposed controller and adaptive gains are presented. Various simulation and experimental studies are conducted to validate the feasibility and performance of the developed method under different initial conditions, quasi-static scenarios, and complex time-varying motion and force profiles. In addition, a comparative study with hybrid control methods applied to rigid manipulators is presented. Compared to the PD controller, our method halves the RMSE ($0.06$ vs. $0.12$) and reduces the STD by $28\%$. Finally, the experimental results demonstrate the controller’s robustness with respect to factors such as variations in the scene, vibrations present in aerial tests, and servoing performance when features are extracted using markers. The controller maintains low interaction forces, $f_{n} < 2\ \mathrm{[N]}$, allowing for safe interaction with environments such as deformable objects. The main challenges in this work were achieving a compact hardware design and balancing the trade-off between dexterity and payload. Future work will extend this framework to multi-arm TD-ACM cooperative manipulation and shape regulation. Moreover, constraint-aware model-free strategies will be examined to further improve robustness.

\bibliography{manuscript}
\bibliographystyle{ieeetr}
\end{document}